\documentclass{article}

\usepackage{verbatim,amsmath,amsfonts,epsfig,graphics,setspace,amssymb, amsbsy, float,stmaryrd, multirow}
\usepackage{caption}
\usepackage{epsfig}
\usepackage{enumerate,color}
\usepackage{url}

\usepackage{algorithm}
\usepackage{algorithmic}
\usepackage{amsthm}
\usepackage{booktabs}
\usepackage{tikz}
\usetikzlibrary{positioning,arrows.meta}

\usepackage[numbers,sort&compress]{natbib}
\bibpunct{[}{]}{,}{n}{,}{--}

\newcounter{rmrk}[section]

\newtheorem{corollary}{Corollary}[section]
\newtheorem{theorem}{Theorem}[section]

\newcommand{\p}{\mathbb{P}}

\newcommand{\B}{\mathcal{B}}

\newcommand{\pspace}{\Lambda}
\newcommand{\dspace}{\mathcal{D}}

\newcommand{\pborel}{\mathcal{B}_{\pspace}}

\newcommand{\dborel}{\mathcal{B}_{\dspace}}

\newcommand{\initmeas}{P_{\text{init}}}
\newcommand{\initmeasomega}{P_{\text{init}, \omega}}
\newcommand{\updatemeas}{P_{\text{up}}}
\newcommand{\initdens}{\pi_{\text{init}}}
\newcommand{\updatedens}{\pi_{\text{up}}}
\newcommand{\predictmeas}{P_\text{pred}}
\newcommand{\predictdens}{\pi_\text{pred}}
\newcommand{\obsmeas}{P_{\text{obs}}}
\newcommand{\obsdens}{\pi_{\text{obs}}}
\newcommand{\obsmeasi}{P_{\text{obs}, i}}

\newcommand{\expnumber}[2]{{#1}\mathrm{E}{#2}}

\newcommand{\mat}[2][cccccccccccc]{\left(\begin{array}{#1}#2\\ \end{array}\right)}

\DeclareMathOperator*{\argmin}{arg\,min}

\title{Copula Transformations for Data-Consistent Inversion}
\author{Troy Butler\thanks{Department of Mathematical and Statistical Sciences, University of Colorado Denver, Denver, CO 80202 ({\tt Troy.Butler@ucdenver.edu})}
\and
Tianyi Jiang\thanks{Department of Statistics, Colorado State University, Fort Collins, CO 80523}
\and João Silva\thanks{Department of Mathematical and Statistical Sciences, University of Colorado Denver, Denver, CO 80202}
  \and
    Harri Hakula\thanks{Department of Mathematics and Systems Analysis, Aalto University, Finland}
    \and
    Timothy Wildey\thanks{Optimization and Uncertainty Quantification Department, Center for Computing Research, Sandia National Labs, Albuquerque, NM 87185.}
    }
\date{{\normalsize This Draft: \today}}

\begin{document}
\maketitle

\begin{abstract}
Data-consistent inversion (DCI) constructs probability measures whose push-forward distributions agree with observed data, while iterative data-consistent inversion (iDCI) extends this framework to generalized stochastic inverse problems by enforcing multiple push-forward constraints sequentially.
Although iDCI avoids the direct approximation of high-dimensional joint densities, its relationship to the original joint DCI solution has remained unclear.
In this work, we establish this relationship through copula theory.
Using Sklar's theorem, we derive a factorization of the DCI update into separate marginal and dependence transformations and show that the discrepancy remaining after convergence of the iDCI algorithm is entirely characterized by the copulas associated with the observed and predicted joint distributions.
This characterization motivates a copula-transformed iDCI solution, and we prove that an exact copula transformation recovers the original DCI solution.
We further establish convergence results for approximate copula transformations under converging sequences of reference measures and progressively enriched feasible sets.
Numerical examples demonstrate how the geometry induced by the quantity-of-interest map governs the importance of the copula transformation, illustrate an adaptive reference-measure refinement strategy for improving computational accuracy under a fixed sampling budget, and demonstrate the progressive refinement of generalized stochastic inverse problems through heterogeneous, asynchronously acquired experiments.
\end{abstract}

\section{Introduction}

Stochastic inverse problems (SIPs) seek to construct probability measures on model parameters that are consistent with observed data while rigorously quantifying the remaining uncertainty.
More broadly, recent perspectives on inverse problems have emphasized probabilistic and measure-centric formulations in which unknowns or
observations may themselves be represented by probability distributions~\cite{NELSEN2026253}.
Recent work has also considered inverse problems posed directly over spaces of probability measures, including variational formulations for recovering input distributions from prescribed push-forward distributions and analyses of stability and regularization for such
SIPs~\cite{LOW+24, LOW+25}.

Within uncertainty quantification, data-consistent inversion (DCI) has emerged as a measure-theoretic framework for solving SIPs by updating an initial probability measure so that its push-forward through a quantity-of-interest (QoI) map agrees with an observed probability measure \cite{BJW18a}.
The earliest form of the density-based solution to the SIP known to the authors can be found in \cite{PR2000} where it was derived through heuristic arguments based on logarithmic pooling and referred to as ``Bayesian melding.''
Over the last decade, the density-based approximation of the DCI solution to the SIP, as derived in \cite{BJW18a} via the Disintegration Theorem \cite{changandpollard1997}, has been further developed, analyzed, and utilized across a variety of applications, e.g., see~\cite{ZM2023, RPK+23, tran2021solving, BGW2020, FBB25, KE25}.
Subsequent developments have established optimality properties of the DCI solution, clarified its relationship with Bayesian inference, extended the framework to learning low-dimensional QoI maps from high-dimensional data, and demonstrated its effectiveness across a variety of engineering applications; see, for example, \cite{BWY20, PCT+23, BBE24, BH20, MSB+22, RHB25} and the references therein.
More recently, iterative data-consistent inversion (iDCI) generalized the framework to generalized stochastic inverse problems (GSIPs) involving multiple push-forward constraints by replacing a single high-dimensional density estimation problem with a sequence of lower-dimensional DCI updates \cite{JBW+26}.
While this sequential construction substantially reduces the computational burden associated with estimating multivariate output densities, the precise relationship between the limiting iDCI solution and the original joint DCI solution has remained unresolved.

The purpose of this work is to establish this relationship through copula theory via Sklar's theorem \cite{Nelsen2006_copulas}.
Using Sklar's theorem, we derive a factorization of the DCI update into separate marginal and dependence transformations, providing two equivalent interpretations of the DCI solution: one based on direct approximation of the observed and predicted joint densities, and one based on separate approximations of their marginal densities and copulas.
Since iDCI already constructs the marginal transformations through successive applications of DCI to the individual push-forward constraints, the only information missing from the full DCI solution is the dependence associated with the copulas.
This characterization motivates a copula-transformed iDCI (CT+iDCI) solution in which the limiting iDCI measure serves as the initial distribution for a final DCI update involving only the remaining dependence structure.
Building on this perspective, we prove that an exact copula transformation recovers the original DCI solution and develop convergence theory for approximate copula transformations under converging reference measures and progressively enriched feasible sets.

The main contributions of this work are as follows:
\begin{itemize}
    \item a copula-based characterization of the relationship between DCI and iDCI through a factorization of the DCI update into separate marginal and dependence transformations;
    \item an information-theoretic characterization of the remaining discrepancy after iDCI together with a proof that exact copula transformation recovers the original DCI solution;
    \item convergence results for approximate CT+iDCI solutions under converging reference measures and nested feasible sets;
    \item numerical examples showing how QoI geometry governs the importance of the copula transformation, how adaptive reference-measure refinement improves approximation quality under a fixed computational budget, and how progressively enriched GSIPs can incorporate heterogeneous, asynchronously acquired experiments.
\end{itemize}

The remainder of the paper is organized as follows. Section~\ref{sec:DCI_background} reviews the SIP, DCI, and iDCI frameworks and establishes the notation used throughout.
Section~\ref{sec:copulas} introduces the copula decomposition and the notation needed to compare observed and predicted densities.
This section also develops the marginal-copula factorization of the DCI update, shows that exact copula transformation of the limiting iDCI solution recovers the original DCI solution, and identifies the exact target for the approximate methodology.
Section~\ref{sec:iDCI_sequences} analyzes convergence of approximate copula-transformed iDCI solutions under varying reference measures and nested feasible sets.
Section~\ref{sec:comp_approx} summarizes, at a high-level, the computational approximation of DCI and iDCI solutions as well as the copula transformations.
Section~\ref{sec:numerics} presents the numerical examples, and Section~\ref{sec:conclusions} concludes the paper and discusses future work.
Section~\ref{sec:supplementary} describes how to obtain the code and datasets necessary to reproduce all of the results presented in the numerical examples.

\section{Background on Data-Consistent Inversion (DCI)}\label{sec:DCI_background}

Broadly, the aim of the SIP is to construct a probability measure on an input space subject to a push-forward constraint on a specified output space.
The GSIP extends the SIP to multiple push-forward constraints on various output spaces.
In \cite{BJW18a}, a density-based approach to data-consistent inversion (DCI) produces a unique (up to choice of initial distribution) and numerically stable solution to the SIP.
As shown in \cite{JBW+26}, an iterative DCI (iDCI) approach solves the GSIP.
Below, we introduce the necessary notation and terminology to define these problems precisely before we summarize their solutions via the DCI and iDCI frameworks.

\subsection{Notation, Terminology, and Problem Definitions}

We first consider the separable, complete metric spaces $\Lambda\subset\mathbb{R}^p$ and $\{\mathcal{D}_{i}\}_{i \le k}$ for a positive integer $k$ denoting the number of separate push-forward constraints, where, for each $i\le k$, $\mathcal{D}_i\subset\mathbb{R}^{d_i}$ for some positive integer $d_i$.
The space $\Lambda$ serves as the input space, which we often refer to as the {\em parameter space} for a simulation model, and the spaces $\{\mathcal{D}_i\}_{i\le k}$ serve as the output spaces, which we often refer to as the {\em observable spaces} defined in terms of measurable data associated with the outputs of the simulation model.
Let $(\Lambda, \B_{\Lambda})$ and $\{(\mathcal{D}_{i}, \B_{\mathcal{D}_{i}})\}_{i \le k}$ be the corresponding measurable spaces with Borel $\sigma$-algebras $\B_{\Lambda}$ and $\B_{\mathcal{D}_{i}}$ for $i\le k$, respectively.
For each $i$, let $\phi_{i}: \Lambda \rightarrow \mathcal{D}_{i}$ be a measurable mapping.
We often refer to each $\phi_i$ as a {\em quantities of interest (QoI) map}, where the plural {\em quantities} emphasizes that $\phi_i$ may be vector-valued with $d_i>1$.

For a given collection of observed probability measures, denoted by $P_{\text{obs},i}$ for each $i\le k$, defined on each $(\mathcal{D}_i, \B_{\mathcal{D}_i})$, the GSIP is defined as finding a \emph{single} probability measure $P$ satisfying the following $k$ push-forward constraints
\begin{equation}\label{eq:k-data-consistent}
    (P\circ \phi_{i}^{-1})(A) = P_{\text{obs},i}(A), \quad \forall A\in\B_{\mathcal{D}_i}, \quad i =1, \ldots, k.
\end{equation}
The push-forward constraints given by~\eqref{eq:k-data-consistent} are referred to as the {\em data-consistency constraints} since, for a given $A\in\dborel$, the solution must assign the same probabilities to events $\phi_i^{-1}(A)\in\B_{\pspace}$ as observed in the data spaces.
Note that any probability measure $P$ on $(\Lambda, \B_\Lambda)$ induces push-forward probability measures denoted by $P\circ \phi_{i}^{-1}$ on $(\mathcal{D}_i, \B_{\mathcal{D}_i})$ for each $i$.
Subsequently, if each $P_{\text{obs},i}$ is the induced push-forward measure of some probability measure defined on $(\Lambda, \B_\Lambda)$, then a solution to the GSIP is guaranteed to exist.
We make such an assumption in this work and refer to this measure as the {\em data-generating distribution} (denoted as $P_\text{DG}$), given its role in inducing the observed probability measures.

With the GSIP defined above, the SIP can then be simply defined as a GSIP with a single push-forward constraint associated with a fixed (often vector-valued) QoI map.
When a single QoI map is considered, we drop the subscripts and refer to the QoI map defined by $\phi$ and the associated measurable observed space $(\mathcal{D},\mathcal{B}_\mathcal{D})$.
It is worth noting that if $\phi$ is vector-valued so that $\mathcal{D}\subset\mathbb{R}^d$ for $d>1$, then a GSIP can be defined in terms of the component maps of $\phi$ and subspaces of $\mathcal{D}$, which are then denoted as $\phi_i$ and $\mathcal{D}_i$, respectively, for $1\leq i\leq d$.

\subsection{Solutions to the SIP and GSIP}

To help distinguish this methodology from classical Bayesian inference, the DCI terminology, which is utilized in this work, was introduced in \cite{BJW18b}.
In that work, an initial and predicted density (denoted by $\initdens$ and $\predictdens$, resp.) are used to describe the initial quantification of uncertainties on $(\pspace,\pborel)$ and $(\dspace,\dborel)$ (resp.) independent of any observed data.
An observed density, denoted by $\obsdens$, then describes the quantification of uncertainty for the observed output data.
The DCI solution is then obtained via the product of $\initdens$ with the ratio of $\obsdens$ to $\predictdens$ evaluated at $\phi(\lambda)$ for each $\lambda\in\pspace$.
We refer to this solution as the updated density and denote it by $\updatedens$.
The rest of this subsection summarizes, briefly, the precise mathematical content needed to construct the DCI solution to the SIP (Section~\ref{subsec:DCI+SIP}), the optimality of the solution in terms of forward and backward minimization with respect to $f$-divergences (Section~\ref{subsec:DCI+Optimality}), and how this optimality informs the iDCI approach to solving the GSIP (Section~\ref{subsec:iDCI+GSIP}).
For a thorough comparison of DCI and Bayesian frameworks and methods, we direct the interested reader to Sections 2 and 4 of \cite{PCT+23}, Section 7 in \cite{BJW18a}, Sections 1 and 2 of \cite{BWY20}, and to a recent review paper \cite{BBE24}.

\subsubsection{Disintegration and DCI Solution to SIP}\label{subsec:DCI+SIP}

Given a probability measure $P^*$ on $(\Lambda,\B_{\Lambda})$ and a measurable mapping $\phi: \Lambda \rightarrow \mathcal{D}$, the disintegration theorem \cite{changandpollard1997} allows us to disintegrate $P^*$ as
\begin{equation} \label{disintegration}
    P^*(A) = \int_{\mathcal{D}}P^*_{\omega}(A) dQ^{*}(\omega), \ \forall \, A\in B_\Lambda,
\end{equation}
where $Q^{*}:= P^{*}\circ\phi^{-1}$ is the push-forward measure of $P^*$ defined on $(\dspace, \B_\dspace)$.
Here, the $\{P^*_\omega\}_{\omega \in \mathcal{D}}$ represent the {\em disintegration} of $P^*$ into a $Q^*$-a.e.~uniquely defined family of probability measures representing conditional probability measures defined on $\phi^{-1}(\omega)$ for a.e.~$\omega\in\mathcal{D}$, i.e., $P^*_\omega(A) = P^*_\omega(A\cap \phi^{-1}(\omega))$.
It follows that if $\initmeas$ denotes a probability measure on $(\Lambda,\B_\Lambda)$ and $\predictmeas:=\initmeas\circ \phi^{-1}$ denotes its pushforward, then the disintegration of $\initmeas$ yields
\begin{equation}\label{eq:disintegrate_initial}
    \initmeas(A) = \int_\dspace \initmeasomega(A)\, d\predictmeas(\omega), \ \forall \, A\in B_\pspace.
\end{equation}
It is shown in \cite{BET+14} that for a given $P_\text{obs}$ on $(\mathcal{D},\B_\mathcal{D})$ and $P_\text{init}$ on $(\Lambda,\B_\Lambda)$ (referred to as an ansatz in that work), the DCI solution is given by an updated measure defined by
\begin{equation} \label{integration}
    P_\text{up}(\cdot) := \int_{\mathcal{D}} P_{\text{init}, \omega}(\cdot) dP_\text{obs}(\omega).
\end{equation}

We further assume that $\initmeas, \predictmeas,$ and $\obsmeas$ are all absolutely continuous with respect to their associated dominating measures (denoted by $\mu_\Lambda$ on $(\Lambda,\B_\Lambda)$ and $\mu_\mathcal{D}$ on $(\mathcal{D}, \B_\mathcal{D})$) so that they admit Radon-Nikodym derivatives denoted by $\initdens$, $\predictdens$, and $\obsdens$, respectively.
Following \cite{BJW18a}, the {\em predictability assumption} states that if there exists a constant $C>0$ such that $\obsdens(\omega)\leq C\predictdens(\omega)$ for a.e.~$\omega$, then the Radon-Nikodym derivative of the updated measure exists and is given by
\begin{equation}\label{eq:updated_density}
    \updatedens(\lambda) = \initdens(\lambda)\frac{\obsdens(\phi(\lambda))}{\predictdens(\phi(\lambda))}.
\end{equation}
Subsequently, the evaluation of $\updatemeas(A)$ for $A\in\B_\Lambda$ can be rewritten either as
\begin{equation}\label{eq:density-based-DCI}
    \updatemeas(A) = \int_\mathcal{D} \int_{A\cap \phi^{-1}(\omega)} \initdens(\lambda)\frac{\obsdens(\phi(\lambda))}{\predictdens(\phi(\lambda))}\, d\mu_{\Lambda,\omega}(\lambda)\,  d\mu_\mathcal{D}(\omega),
\end{equation}
where $\{\mu_{\pspace,\omega}\}_{\omega\in\dspace}$ comes from the disintegration of $\mu_\pspace$, or as
\begin{equation}\label{eq:density-based-DCI-factored}
    \updatemeas(A) = \int_\mathcal{D} \left(\int_{A\cap \phi^{-1}(\omega)} \initdens(\lambda)\, d\mu_{\Lambda,\omega}(\lambda)\right) \frac{\obsdens(\omega)}{\predictdens(\omega)}\,  d\mu_\mathcal{D}(\omega)
\end{equation}
This latter form emphasizes that the DCI solution is conceptually defined as a re-weighting of the initial measure, but only in the directions of $\Lambda$ for which $\phi$ varies.
While~\eqref{eq:updated_density}-\eqref{eq:density-based-DCI-factored} are all presented for Radon-Nikodym derivatives, we generally refer to these as representing the ``density-based'' DCI solution.
Beyond the existence and uniqueness (up to the choice of initial distribution) of the DCI solution, \cite{BJW18a} also proves stability in the total-variation (TV) metric with respect to perturbations in the initial, observed, and predicted distributions.

\subsubsection{Optimality of the DCI Solution}\label{subsec:DCI+Optimality}

Besides generalizing many well-known divergence measures, such as the Kullback-Leibler (KL) divergence and TV metric \cite{KL1951, RenyiKL2014, BKM17, AH21}, $f$-divergences \cite{polyanskiy2022information} are commonly employed in the context of determining optimal parameters or hyper-parameters of density models to quantify the difference between probability measures, e.g., see~\cite{Guntuboyina2011, BKM17}.
Utilizing the notation of this work, the $f$-divergence we are concerned with is defined as:
\[
    D_{f}(P || P_\text{init}) = \int_\pspace f\left(\frac{dP}{dP_\text{init}}(\lambda)\right) dP_\text{init}(\lambda)
\]
where $P \ll P_\text{init}$ and \( f: [0, \infty) \to \mathbb{R} \) is a convex function satisfying \( f(1) = 0 \) and $f(0) = \lim_{t \rightarrow 0^{+}}f(t)$.
When $f(t)=t\log t$, we obtain the KL divergence denoted by $D_{\rm KL}$ in this work.

Let $\p = \{P\ll \initmeas\, : \, P\circ \phi^{-1} = \obsmeas\}$ denote the space of all pullback measures of $\obsmeas$ that are absolutely continuous with respect to $\initmeas$.
Under the predictability assumption, \cite{JBW+26} proved that the DCI solution is the unique minimizer of the following forward and backward minimization problems:
\[
    \updatemeas = \argmin_{P\in \p} D_{f}(P || \initmeas) \quad \text{Forward minimization},
\]
and, if $\initmeas \ll \updatemeas$, then
\[
    \updatemeas = \argmin_{P\in \tilde{\p}} D_{\rm KL}(\initmeas || P) \quad \text{Backward minimization}
\]
where $\tilde{\p}$ is the subspace of $\p$ such that $\initmeas \ll P$ for all $P\in\tilde{\p}$.

\subsubsection{The iDCI Solution to the GSIP}\label{subsec:iDCI+GSIP}

Following the notation of~\cite{JBW+26}, for each $1\leq i\leq k$, let $\p_{i}$ denote the set of probability measures that satisfy the $i$th push-forward constraint, i.e., $\p_{i} := \{P\ll \initmeas\, :\,  P\circ\phi_{i}^{-1} = P_{\text{obs},i}\}$.
In the iDCI approach, we denote $\p_{mk+1}=\p_1, \p_{mk+2}=\p_2, \ldots$, where $m=0, 1, \ldots$.
It follows that $\p_n=\p_i$ if $(n-i) \mathrm{~mod~} k =0$ (i.e., if there exists nonnegative integer $m$ such that $n=mk+i$ for $1\leq i\leq k$).
The solution set of the GSIP is then expressed as
\begin{equation*}
    \p = \bigcap_{i = 1}^k \p_{i} = \bigcap_{n = 1}^\infty \p_{n}.
\end{equation*}
Since each set $\p_i$ is convex, so is the intersection defining $\p$.
Moreover, under the assumption that the $P_{\text{obs},i}$ are themselves induced by some data-generating distribution, the above intersection is non-empty as it must contain at least $P_\text{DG}$.
The solution to the GSIP is then formulated as the following optimization problem over a convex set:
\begin{align} \label{overall_min}
    \min_{P \in \p}D_f(P|| P_\text{init}).
\end{align}
Clearly, all elements in $\p$ can be viewed as a minimizer for some choice of $P_\text{init}$.
However, we emphasize that $P_\text{init}$ need not belong to $\p$.
This particular probability measure simply represents the initial probability measure previously mentioned and is a place where expert opinion can weigh in on the final structure of the solution.
The only requirement on $\initmeas$ is that $\obsmeasi \ll \initmeas\circ \phi_i^{-1}$ for each $i$.

In \cite{JBW+26}, the iDCI approach is defined by considering one constraint at a time to define a sequence of subproblems that we iterate through to compute a solution.
The $n$th subproblem is defined as the minimization problem
\begin{align} \label{sub_min}
    \min_{P \in \p_{n}} D_f(P ||P^{n-1}),
\end{align}
where $P^{n-1}$ is the minimizer of the previous subproblem.
By the optimality of the DCI solution, the minimum of~\eqref{sub_min} is given by the DCI updated measure denoted by $P^{n}$ where $P^{n-1}$ serves the role of the initial distribution, QoI map $\phi_i$ (where $i=n \mathrm{~mod~} k$) is utilized to construct the associated predicted distribution, and $\obsmeasi$ is the observed measure.
Thus, to construct the iDCI solution for a given initial probability measure $P_\text{init}$, we set $P^0=P_\text{init}$ and iteratively solve the subproblems within the DCI framework to yield a sequence of local solutions, $P^1, P^2, \ldots$.
Under the conditions that (i) there exists at least one $P \in \p$ such that $P \ll P^{0}$, i.e. $\inf_{P \in \p} D_{\rm KL}(P || P^{0}) < \infty$, and (ii) a variant of the predictability assumption holds for each constraint,
then \cite{JBW+26} proves that $P^{n}$ converges in total variation to $P^{\infty} \in \p$ where $P^{\infty}$ is the information projection ($I$-projection) of $P^{0}$ onto $\p$, i.e.,
\[
P^{\infty} := \argmin_{P \in \p} D_{\rm KL}(P || P^{0}).
\]

\section{Copulas and the Relationship between DCI and iDCI}\label{sec:copulas}

In practice, accurately estimating multivariate probability densities is often the most computationally challenging aspect of implementing DCI, particularly as the dimension of the observable space increases.
In contrast, the iDCI solution to the associated GSIP defined entirely in terms of the component maps of the QoI requires approximating  only marginal densities corresponding to individual push-forward constraints.
This naturally raises the question of how the DCI and iDCI solutions are related when the QoI maps for the GSIP are defined entirely in terms of the component maps of a vector-valued QoI.
Sklar's theorem provides precisely the mathematical framework necessary to separate marginal behavior from dependence through a copula decomposition \cite{Nelsen2006_copulas}.

\subsection{Copulas and Sklar's Theorem}\label{subsec:Sklar}

To develop a clear understanding of the relationship between the DCI and iDCI solutions, we restrict attention to the GSIP obtained from the component maps of the vector-valued QoI defining the SIP.
This restriction is made solely to simplify the exposition through the utilization of the classical form of Sklar's theorem.
The component-map formulation considered here represents the most refined decomposition of the observable space and provides the clearest setting in which to expose the relationship between DCI, iDCI, and copula transformations.
Several examples of more general collections of QoI maps are discussed later in Section~\ref{subsec:nested-feasible} in the context of progressively enriching the associated feasible sets.
More generally, the GSIP may be formulated using an arbitrary collection of scalar- or vector-valued QoI maps, including maps associated with overlapping or nested subspaces of the observable space.
In that setting, not only does each of the vector-valued QoI admit its own copula decomposition through the classical form of Sklar's theorem, but the application of a more recently analyzed vector-valued form of Sklar's theorem \cite{FH23_VectorCopulas} immediately extends the analysis presented below to this more general setting.

For notational simplicity, throughout this section, we let
\[
\phi=(\phi_1,\ldots,\phi_d):\Lambda\to D\subset\mathbb R^d
\]
denote the vector-valued QoI map used in DCI and assume its components are used in iDCI.
We assume throughout that the observed and predicted probability measures on $(\dspace,\B_\dspace)$ are absolutely continuous with respect to $\mu_D$ and admit densities $\pi_{\rm obs}$ and $\pi_{\rm pred}$, respectively, and that the predictability assumption is satisfied.
Furthermore, let $\pi_{{\rm obs},i}$ and $\pi_{{\rm pred},i}$ denote the corresponding marginal densities for each component $1\le i\le d$, and let $F_{{\rm obs},i}$ and $F_{{\rm pred},i}$
denote their cumulative distribution functions (CDFs).
The following classical result of Sklar provides the fundamental decomposition that motivates the remainder of this paper.

\begin{theorem}[Sklar's Theorem]
Suppose the CDFs $F_1,\ldots,F_d$
are continuous.
Then there exists a unique copula
$C:[0,1]^d\to[0,1]$ such that the corresponding joint distribution function satisfies
\[
F(y_1,\ldots,y_d)
=
C(F_1(y_1),\ldots,F_d(y_d)).
\]
Moreover, if the joint distribution is absolutely continuous, then the copula admits a density $c$ satisfying
\[
\pi(y)
=
c(F_1(y_1),\ldots,F_d(y_d))
\prod_{i=1}^d
\pi_i(y_i).
\]
\end{theorem}

Applying Sklar's theorem to the observed and predicted distributions separately yields unique copula densities, denoted by $c_{\rm obs}$ and $c_{\rm pred}$, satisfying

\begin{align}
\pi_{\rm obs}(y)
&=
c_{\rm obs}\!\left(
F_{{\rm obs},1}(y_1),
\ldots,
F_{{\rm obs},d}(y_d)
\right)
\prod_{i=1}^d
\pi_{{\rm obs},i}(y_i),
\label{eq:obs-copula-factorization}
\\
\pi_{\rm pred}(y)
&=
c_{\rm pred}\!\left(
F_{{\rm pred},1}(y_1),
\ldots,
F_{{\rm pred},d}(y_d)
\right)
\prod_{i=1}^d
\pi_{{\rm pred},i}(y_i).
\label{eq:pred-copula-factorization}
\end{align}

Equations~\eqref{eq:obs-copula-factorization}--\eqref{eq:pred-copula-factorization} express each joint density as the product of two distinct components.
The marginal densities determine the uncertainty associated with each observable individually, while the copula density contains all information describing the dependence among the observable quantities.
Consequently, two multivariate distributions possess identical dependence structures if and only if they share the same copula.

\subsection{A Marginal--Copula Factorization of the DCI Update}

For notational convenience, define the probability integral transforms associated with the observed and predicted distributions by
\begin{align}
U_{\rm obs}(y)
&:=
\left(
F_{{\rm obs},1}(y_1),
\ldots,
F_{{\rm obs},d}(y_d)
\right),
\label{eq:Uobs}
\\
U_{\rm pred}(y)
&:=
\left(
F_{{\rm pred},1}(y_1),
\ldots,
F_{{\rm pred},d}(y_d)
\right).
\label{eq:Upred}
\end{align}
The following theorem establishes an equivalent factorization of the DCI update into separate marginal and dependence transformations with respect to these copula coordinate systems.

\begin{theorem}[Marginal--Copula Factorization of the DCI Update]
Suppose the assumptions of Section~\ref{subsec:Sklar} hold. Then, the density-based DCI solution admits the representation
\begin{align}
\pi_{\rm up}(\lambda)
&=
\pi_{\rm init}(\lambda)
\left(
\prod_{i=1}^{d}
\frac{\pi_{{\rm obs},i}(\phi_i(\lambda))}
     {\pi_{{\rm pred},i}(\phi_i(\lambda))}
\right)\times
\frac{
c_{\rm obs}(U_{\rm obs}(\phi(\lambda)))
}{
c_{\rm pred}(U_{\rm pred}(\phi(\lambda)))
}.
\label{eq:dci-factorization-updated}
\end{align}

\end{theorem}

\begin{proof}
Substituting the copula factorizations
\eqref{eq:obs-copula-factorization}--\eqref{eq:pred-copula-factorization}
into the ratio $\pi_{\rm obs}(y)/\pi_{\rm pred}(y)$
yields

\[
\frac{\pi_{\rm obs}(y)}
     {\pi_{\rm pred}(y)}
=
\frac{
c_{\rm obs}(U_{\rm obs}(y))
\prod_{i=1}^{d}\pi_{{\rm obs},i}(y_i)
}{
c_{\rm pred}(U_{\rm pred}(y))
\prod_{i=1}^{d}\pi_{{\rm pred},i}(y_i)
}.
\]
Separating the products, evaluating this identity at
$y=\phi(\lambda)$,
and substituting the result into the density-based DCI solution~\eqref{eq:density-based-DCI} yields~\eqref{eq:dci-factorization-updated}.
\end{proof}

Equation~\eqref{eq:dci-factorization-updated} provides two mathematically equivalent factorizations of the DCI update.
The standard formulation evaluates the ratio of the observed and predicted joint densities directly, whereas the factorized representation separates this transformation into a product of marginal density ratios together with a transformation that depends exclusively upon the dependence structures of the observed and predicted distributions.

A subtle but important feature of the factorized representation is that the copula ratio is generally evaluated at two different probability integral transforms.
Specifically, the observed copula is evaluated at $U_{\rm obs}$ while the predicted copula is evaluated at $U_{\rm pred}$.
These mappings coincide if and only if the observed and predicted marginal distributions agree.
Consequently, the DCI update simultaneously transforms both the marginal distributions and the dependence structure of the predicted distribution.
The dependence transformation therefore reflects not only differences between the observed and predicted copulas, but also differences between the coordinate systems on which those copulas are defined.

\subsection{The Copula Discrepancy Remaining After iDCI}

The principal distinction between DCI and iDCI is that the limiting iDCI solution satisfies each marginal push-forward constraint individually.
Consequently, the DCI factorization collapses to a common coordinate system, leaving only a discrepancy in the associated copulas.

\begin{theorem}[Copula representation of the limiting iDCI push-forward]\label{thm:copula+iDCI}
Suppose the assumptions of Section~2.2.3 hold so that the iDCI sequence converges to the limiting measure
$P^\infty$, and let $P^\infty_{\rm pred}:=P^\infty\circ\phi^{-1}$
denote its joint push-forward through the vector-valued QoI map.
Then, the associated density admits the representation
\[
\pi^\infty_{\rm pred}(y)
=
c_\infty(U_{\rm obs}(y))
\prod_{i=1}^{d}
\pi_{{\rm obs},i}(y_i),
\]
for a unique copula density
$c_\infty$.
Equivalently, $U_\infty(y)=U_{\rm obs}(y)$, where $U_\infty$
denotes the probability integral transform associated with
$\pi^\infty_{\rm pred}$.

\end{theorem}

\begin{proof}
By construction, the limiting iDCI solution satisfies every marginal push-forward constraint, i.e., $P^\infty\circ\phi_i^{-1}=P_{{\rm obs},i}$ for $i=1,\ldots,d$.
Hence, $\pi^\infty_{{\rm pred},i}=
\pi_{{\rm obs},i}$ for $i=1,\ldots,d$.
It follows that both $F^\infty_i=
F_{{\rm obs},i}$ and $U_\infty(y)=
U_{\rm obs}(y)$ for $i=1,\ldots, d$.

Applying Sklar's theorem to the joint push-forward
$P^\infty_{\rm pred}$
gives
\[
\pi^\infty_{\rm pred}(y)
=
c_\infty(U_\infty(y))
\prod_i
\pi^\infty_{{\rm pred},i}(y_i),
\]
which simplifies to the stated expression after substituting the equalities above.
\end{proof}

\subsection{An Information-Theoretic Characterization of the Remaining Discrepancy}

Theorem~\ref{thm:copula+iDCI} establishes that, after convergence of the iDCI algorithm, the observed and predicted joint distributions possess identical marginal distributions and therefore share a common probability integral transform defined on a common coordinate system.
Consequently, the remaining discrepancy between the two distributions is entirely contained within their copulas. We now show that this observation admits a precise information-theoretic interpretation through the KL divergence.

\begin{theorem}[Copula divergence]\label{thm:KL+copula+iDCI}
Suppose the assumptions of Theorem~\ref{thm:copula+iDCI} hold.
Then,
\begin{equation}\label{eq:iDCI_ratio_copula}
\frac{\pi_{\rm obs}(y)}
     {\pi_{\rm pred}^{\infty}(y)}
=
\frac{
c_{\rm obs}(U_{\rm obs}(y))
}{
c_{\infty}(U_{\rm obs}(y)),
}
\end{equation}
and
\begin{equation}
D_{\rm KL}
\!\left(
\pi_{\rm obs}
\,\middle\|\,
\pi_{\rm pred}^{\infty}
\right)
=
D_{\rm KL}
\!\left(
c_{\rm obs}
\,\middle\|\,
c_{\infty}
\right),
\label{eq:copula-kl}
\end{equation}
where the copula divergence is taken with respect to Lebesgue measure on $[0,1]^d$.
\end{theorem}

Before proving this result, we first note that the KL divergence appearing above is well-defined whenever $P_{\rm obs}
\ll
P_{\rm pred}^{\infty}$,
which is precisely the natural predictability condition associated with using the limiting iDCI push-forward as the initial prediction for a final DCI update.
In particular, this assumption is the direct analogue of the predictability assumption required in the original DCI solution.
Under this assumption, the Radon-Nikodym derivative
\[
\frac{dP_{\rm obs}}
     {dP_{\rm pred}^{\infty}}
\]
exists almost everywhere, allowing the Kullback--Leibler divergence to be evaluated.

\begin{proof}
Equation~\eqref{eq:iDCI_ratio_copula} follows immediately from  Sklar's theorem and Theorem~\ref{thm:copula+iDCI}.
Substituting this identity into the KL divergence yields

\begin{align}
D_{\rm KL}
\!\left(
\pi_{\rm obs}
\,\middle\|\,
\pi_{\rm pred}^{\infty}
\right)
&=
\int_\dspace
\pi_{\rm obs}(y)
\log
\frac{
c_{\rm obs}(U_{\rm obs}(y))
}{
c_{\infty}(U_{\rm obs}(y))
}
\,d\mu_\dspace(y).
\label{eq:KL-step1}
\end{align}

Let $C_{\rm obs}:=
P_{\rm obs}\circ U_{\rm obs}^{-1}
$ denote the copula measure associated with the observed distribution.
By construction, $C_{\rm obs}$ possesses density $c_{\rm obs}$ with respect to Lebesgue measure on $[0,1]^d$.
Therefore, for every integrable function
$f:[0,1]^d\rightarrow\mathbb R$, we are able to change integrals of $f(U_{\rm obs}(y))$ over $\dspace$ to integrals over $[0,1]^d$ as follows,
\[
\int_\dspace
f(U_{\rm obs}(y))
\,dP_{\rm obs}(y)
=
\int_{[0,1]^d}
f(u)
\,dC_{\rm obs}(u)
=
\int_{[0,1]^d}
f(u)
c_{\rm obs}(u)
\,du.
\]
Applying the above identity to
\[
f(u)
=
\log
\frac{c_{\rm obs}(u)}
     {c_{\infty}(u)}
\]
transforms~\eqref{eq:KL-step1}
into
\begin{align*}
D_{\rm KL}
\!\left(
\pi_{\rm obs}
\,\middle\|\,
\pi_{\rm pred}^{\infty}
\right)
&=
\int_{[0,1]^d}
c_{\rm obs}(u)
\log
\frac{
c_{\rm obs}(u)
}{
c_{\infty}(u)
}
\,du \\
&=
D_{\rm KL}
\!\left(
c_{\rm obs}
\,\middle\|\,
c_{\infty}
\right),
\end{align*}
which proves
\eqref{eq:copula-kl}.
\end{proof}

Equation~\eqref{eq:copula-kl} therefore provides rigorous justification for interpreting the remaining update after iDCI as a \emph{copula transformation}. Rather than approximating an entire multivariate density ratio, it suffices to estimate the copula associated with the limiting iDCI push-forward and compare it with the observed copula.

\subsection{Copula-transformed iDCI as a Second DCI Projection}\label{subsec:cc-iDCI}

A natural question is whether a transformation based on the copula discrepancy in the iDCI solution recovers the original DCI solution.
In this subsection, we show that the resulting copula-transformed iDCI solution is itself a DCI solution, but with the limiting iDCI distribution serving as the initial distribution.
Consequently, the original DCI solution and the copula-transformed iDCI solution solve the same constrained optimization problem over the same feasible set, but with different reference measures. As shown below, this observation permits the application of classical projection results for information divergences and ultimately establishes that an exact copula transformation recovers the original DCI solution.

Motivated by Theorem~\ref{thm:KL+copula+iDCI}, we define the \emph{copula-transformed iDCI solution} by applying a single DCI update using $P^\infty$ as the initial measure and the observed joint distribution as the target.
Equivalently, the corresponding copula-transformed density is defined by
\begin{equation}
\pi_{\rm CT}(\lambda)
:=
\pi^\infty(\lambda)
\frac{\pi_{\rm obs}(\phi(\lambda))}
     {\pi_{\rm pred}^\infty(\phi(\lambda))}
=
\pi^\infty(\lambda)
\frac{
c_{\rm obs}(U_{\rm obs}(\phi(\lambda)))
}{
c_\infty(U_{\rm obs}(\phi(\lambda)))
},
\label{eq:copula-transformed}
\end{equation}
where the second equality follows from Theorem~\ref{thm:KL+copula+iDCI}.

\begin{corollary}[Copula-transformed iDCI as a DCI projection]\label{thm:cc_iDCI}

Let $
\tilde{\p}
:=
\left\{
P\ll P^0:
P\circ\phi^{-1}
=
P_{\rm obs}
\right\}
$
denote the feasible set defining the original DCI problem.
Then, the copula-transformed iDCI solution satisfies
\[
P_{\rm CT}
=
\operatorname*{arg\,min}_{P\in \tilde{\p}}
D_f(P\|P^\infty), \quad \text{ and } \quad
P_{\rm CT} = \argmin_{P\in \tilde{\p}} D_{\rm KL}(P^\infty || P)
\]
\end{corollary}

\begin{proof}
Theorem~\ref{thm:KL+copula+iDCI} shows that the copula-transformed update
\eqref{eq:copula-transformed}
is precisely the density-based DCI solution obtained by taking
$P^\infty$
as the initial probability measure and
$P_{\rm obs}$
as the observed probability measure associated with the vector-valued QoI map.
The optimality of the DCI solution in terms of its forward minimization of $f$-divergences and backward minimization of the KL divergence immediately apply. \end{proof}

We note that the result above follows immediately from the optimality characterization of the DCI solution established in Section~\ref{subsec:DCI+Optimality}.
Moreover, if we let $P_{\rm DCI}$ denote the original DCI solution, then the following inequalities are immediate consequences of the optimality of DCI solutions over $f$-divergences,
\[
D_f(P_{\rm DCI}\|P^0)
\le
D_f(P_{\rm CT}\|P^0),
\quad \text{ and } \quad
D_f(P_{\rm CT}\|P^\infty)
\le
D_f(P_{\rm DCI}\|P^\infty).
\]

The optimality characterization in Corollary~\ref{thm:cc_iDCI} immediately suggests a stronger question: does performing an exact copula transformation after convergence of the iDCI algorithm recover precisely the same probability measure as applying DCI directly to the original SIP? The following theorem answers this question affirmatively by appealing to a classical identity for successive information projections due to Csiszár~\cite{Csiszar1975}.

\begin{theorem}[Copula-transformed iDCI equivalent to the original DCI]\label{thm:cc_i_equals_DCI}$P_{\rm DCI} = P_{\rm CT}$.
\end{theorem}

\begin{proof}
Since every probability measure satisfying the joint push-forward constraint necessarily satisfies each marginal push-forward constraint, it immediately follows that $
\widetilde{\p}
\subset
\p$
where $\widetilde{\p}
=
\{
P \ll P^0 :
P\circ\phi^{-1}
=
P_{\rm obs}
\}$ denotes the DCI feasible set and $\p$ denotes the iDCI feasible set.
Consequently, the copula-transformed iDCI solution is obtained through two successive $I$-projections where the second projection is taken onto a feasible set contained in the first.
The remainder of this proof relies upon careful application of prior results from \cite{Csiszar1975} and \cite{JBW+26} as described below.

First, substituting the notation of this work into Theorem~2.3 of~\cite{Csiszar1975}, we have that if an analogue to Pythagoras' identity holds (referred to as equation (1.7) in that work), i.e., if $D_{\rm KL}(P \| P^0) = D_{\rm KL} (P \| P^\infty) + D_{\rm KL} (P^\infty \| P^0)$  for all $P\in \widetilde{\p}$, then $P_{\rm DCI}=P_{\rm CT}$.
This analogue to Pythagoras' identity then follows directly from what is referred to as ``Case (B)'' in Section 3 of~\cite{Csiszar1975} where we substitute the generalization of the marginal density representations of Step~2 of the proof of Theorem~5.1 in~\cite{JBW+26} to achieve the required product of marginals representation for the joint distribution in \cite{Csiszar1975}.
The projection identity therefore yields
$
P_{\rm DCI}
=
P_{\rm CT}.
$
\end{proof}

Theorem~\ref{thm:cc_i_equals_DCI} provides the theoretical justification for viewing the copula-transformed iDCI solution as a computational realization of the original DCI solution.
In particular, when the exact copulas associated with the observed distribution and the limiting iDCI push-forward are available, the copula transformation reproduces the DCI solution exactly.

In the first numerical example of Section~\ref{sec:numerics}, we utilize linear maps with Gaussian distributions where the exact DCI solution is simply the DG distribution. Furthermore, since the copula is Gaussian, the copula approximations obtained through parametric estimates produce nearly exact copula transformations, which illustrates this theoretical result.

However, this equivalence between $P_{\rm DCI}$ and $P_{\rm CT}$ relies fundamentally on the exact copula transformation satisfying the original joint push-forward constraint. Approximate copula transformations constructed from numerical estimation generally do not satisfy this constraint exactly and therefore need not remain projections onto the feasible set $\widetilde{\p}$. Consequently, the projection identity of Theorem~\ref{thm:cc_i_equals_DCI} no longer applies directly, motivating the convergence analysis developed in the next section.

\section{Convergence of Approximate iDCI Solutions}\label{sec:iDCI_sequences}

The prior section identifies the copula-transformed iDCI solution as the exact target obtained by applying a second DCI projection to the limiting iDCI solution.
Since the required copula is generally unavailable, we instead replace this projection by an approximate copula transformation.
Iterating this procedure produces a sequence of approximate DCI projections whose convergence follows from the analysis of this section.

It is worth noting, at a high-level, that the factorization of a joint density in terms of its copula and marginal densities is particularly attractive computationally because the dependence component may be approximated independently of the marginal densities. Consequently, simple parametric copula families, such as Gaussian copulas, can be employed to approximate the dependence transformation without restricting the marginal density models.
A standard corollary of Sklar's theorem that is worth noting is that copulas are invariant under monotone transformations,
which further emphasizes that copulas isolate the dependence structure of a multivariate distribution independently of the marginal distributions.
Consequently, one may combine arbitrary marginal models with an independently chosen copula model to construct a joint distribution.
This observation is particularly useful computationally since simple copula families, such as Gaussian copulas, may provide effective approximations of dependence even when the marginals are highly non-Gaussian.

In this work, we exploit the flexibility of parametric copula estimation to approximate copula models for both the observed and predicted distributions and subsequently define approximate copula-transformed iDCI solutions.
However, such an approximation to the copula-transformed iDCI solution may no longer satisfy all the push-forward constraints.
Subsequently, the approximate copula-transformed iDCI solution should be subjected once again to the iDCI algorithm.
Repeating this process produces
a sequence of approximate copula-transformed iDCI solutions.
The question naturally arises as to whether this process converges.
This motivates the general theoretical framework provided below for guaranteeing convergence under a sequence of either converging initial distributions or nested feasible sets.

\subsection{Converging sequences of initial distributions}\label{subsec:initial-sequences}

Recall that the copula-transformed iDCI solution derived in Section~\ref{subsec:cc-iDCI} represents the exact target of the proposed methodology.
The purpose of this section is not to approximate this solution directly, but rather to show that a sequence of approximate copula transformations to iDCI solutions converges to this exact target.
To that end, we first consider a more general problem where we assume there exists a sequence of probability measures $P_{k}^{0}$ on $\pspace$ that approach a limiting reference measure $P_\infty^0$, and seek conditions under which the corresponding sequence of I-projections converges in the sense of
\[
\argmin_{P \in \p}D_{\rm KL}(P || P_{k}^{0})
\overset{TV}{\longrightarrow}
\argmin_{P \in \p}D_{\rm KL}(P || P_\infty^{0})
\]
for some feasible set $\p$.

\begin{theorem}
\label{thm:mtd1_converges}
Let $\p$ be a nonempty convex feasible set of probability measures on
$(\Lambda,\mathcal{B}_{\Lambda})$, and let $\{P_k^0\}_{k\in\mathbb{N}}$
and $P_\infty^0$ be probability measures for which the corresponding
I-projections onto $\p$ exist with finite KL divergence. Define
\[
P_k^\ast
:=
\arg\min_{P\in\p}
D_{\mathrm{KL}}(P\|P_k^0),
\qquad
P_\infty^\ast
:=
\arg\min_{P\in\p}
D_{\mathrm{KL}}(P\|P_\infty^0).
\]
Suppose
\[
\eta_k
:=
\left\|
\log\left(
\frac{dP_k^0}{dP_\infty^0}
\right)
\right\|_{L^\infty(P_\infty^0)}
\longrightarrow 0.
\]
Then
\[
D_{\mathrm{KL}}(P_k^\ast\|P_\infty^\ast)
\leq 2\eta_k,
\]
and, consequently,
\[
d_{\mathrm{TV}}(P_k^\ast,P_\infty^\ast)
\leq \sqrt{\eta_k}
\longrightarrow 0.
\]
\end{theorem}

The hypothesis of Theorem~\ref{thm:mtd1_converges} provides uniform control of the change in the reference measure on the
log-density scale naturally associated with the KL divergence.
In particular,
the difference between the KL objectives associated with $P_k^0$ and
$P_\infty^0$ can be bounded uniformly over all admissible probability
measures.
The proof combines this observation with the Pythagorean inequality
for I-projections onto convex sets and Pinsker's inequality to obtain the
explicit stability estimate above.

\begin{proof}
Let $F_k(P):=D_{\mathrm{KL}}(P\|P_k^0)$ and $F_\infty(P):=D_{\mathrm{KL}}(P\|P_\infty^0)$.
By the hypothesis, for all sufficiently large $k$,
\[
e^{-\eta_k}
\leq
\frac{dP_k^0}{dP_\infty^0}
\leq
e^{\eta_k}
\qquad
P_\infty^0\text{-a.e.},
\]
so $P_k^0$ and $P_\infty^0$ are mutually absolutely continuous.
Hence, for any probability
measure $P$ for which either $F_k(P)$ or $F_\infty(P)$ is finite,
\[
\begin{aligned}
F_k(P)-F_\infty(P)
&=
\int_\Lambda
\log\left(
\frac{dP}{dP_k^0}
\right)dP
-
\int_\Lambda
\log\left(
\frac{dP}{dP_\infty^0}
\right)dP \\
&=
-\int_\Lambda
\log\left(
\frac{dP_k^0}{dP_\infty^0}
\right)dP.
\end{aligned}
\]
This immediately implies that $\left|F_k(P)-F_\infty(P)\right|\leq \eta_k.$
Combining this with the optimality of $P_k^\ast$ with respect to $F_k$ gives
\begin{align}
F_\infty(P_k^\ast)
&\leq F_k(P_k^\ast)+\eta_k \nonumber \\
&\leq F_k(P_\infty^\ast)+\eta_k \qquad \text{(by optimality of $P_k^\ast$ w.r.t. $F_k$)}\nonumber \\
&\leq F_\infty(P_\infty^\ast)+2\eta_k \label{eq:F_infty_ineq}.
\end{align}

Since $\p$ is convex and $P_\infty^\ast$ is the I-projection of
$P_\infty^0$ onto $\p$, the Pythagorean inequality for the KL divergence
\cite{Csiszar1975} implies, for $P_k^\ast\in\p$,
\[
D_{\mathrm{KL}}(P_k^\ast\|P_\infty^0)
\geq
D_{\mathrm{KL}}(P_k^\ast\|P_\infty^\ast)
+
D_{\mathrm{KL}}(P_\infty^\ast\|P_\infty^0).
\]
Equivalently,
\[
F_\infty(P_k^\ast)
\geq
D_{\mathrm{KL}}(P_k^\ast\|P_\infty^\ast)
+
F_\infty(P_\infty^\ast).
\]
Combining this inequality with~\eqref{eq:F_infty_ineq} yields
\[
D_{\mathrm{KL}}(P_k^\ast\|P_\infty^\ast)
\leq 2\eta_k.
\]
Finally, Pinsker's inequality gives
\[
d_{\mathrm{TV}}(P_k^\ast,P_\infty^\ast)
\leq
\sqrt{
\frac{1}{2}
D_{\mathrm{KL}}(P_k^\ast\|P_\infty^\ast)
}
\leq
\sqrt{\eta_k}.
\]
Since $\eta_k\to0$, it follows that
\[
P_k^\ast
\xrightarrow{\mathrm{TV}}
P_\infty^\ast,
\]
which completes the proof.
\end{proof}

At this point, we find it conceptually useful to introduce operator notation that emphasizes the structure of the proposed computational methodology.
Let
\[
\mathcal T_{\rm iDCI}(P^0)
:=
\arg\min_{P\in\mathcal \p}
D_{\mathrm{KL}}(P\|P^0)
\]
denote the operator induced by the iDCI solution, where $\p$ is the feasible set associated with the GSIP.
With this notation, $P^\infty=\mathcal T_{\rm iDCI}(P^0)$.
Likewise, let
\[
\mathcal T_{\rm CT}(P^\infty)
:=
\arg\min_{P\in\widetilde{\p}}
D_{\mathrm{KL}}(P\|P^\infty),
\]
denote the exact copula transformation operator introduced in Section~\ref{subsec:cc-iDCI}  where $\widetilde{\p}$
denotes the feasible set associated with the SIP defined by the vector-valued QoI map.
With this notation, $P_{\rm CT}=\mathcal{T}_{\rm CT}(P^\infty)$.
Finally, let $\widehat{\mathcal T}_{\rm CT}$
denote the corresponding operator obtained from approximate copula models.

With this notation, the exact copula-transformed iDCI solution is simply
\[
P_{\rm CT}
=
\mathcal T_{\rm CT}
\!\left(
\mathcal T_{\rm iDCI}(P^0)
\right).
\]
This representation emphasizes the successive transformations of an initial measure by the composition of the iDCI procedure followed by the associated (exact) copula transformation, i.e.,
\[
P^0
\;\xrightarrow{\,\mathcal T_{\rm iDCI}\,}\;
P^\infty
\;\xrightarrow{\,\mathcal T_{\rm CT}\,}\;
P_{\rm CT}.
\]
By replacing the exact copula transformation operator, $\mathcal{T}_{\rm CT}$, with its approximation, $\widehat{\mathcal{T}}_{\rm CT}$, we obtain the practical algorithm proposed in this work that generates a sequence of initial probability measures according to
\begin{equation}
P_{k+1}^0
=
\widehat{\mathcal T}_{\rm CT}
\!\left(
\mathcal T_{\rm iDCI}(P_k^0)
\right),
\qquad
k=0,1,2,\ldots,
\label{eq:operator_iteration}
\end{equation}

The operator $\mathcal T_{\rm iDCI}$ first enforces the marginal push-forward constraints to produce a marginally consistent probability measure.
The operator $\widehat{\mathcal T}_{\rm CT}$ then approximately transforms the remaining dependence discrepancy, producing the initial distribution for the next iteration.
With this perspective, equation~\eqref{eq:operator_iteration} emphasizes that the proposed methodology is fundamentally an iteration on the space of initial probability measures rather than on the space of copulas where each iteration involves two successive $I$-projections.

When the sequence generated by
\eqref{eq:operator_iteration}
converges in total variation, we denote its limit by
\[
P^0_\infty
:=
\lim_{k\to\infty}P_k^0,
\]
and refer to $P^0_\infty$ as the \emph{self-consistent probability measure} associated with the iterative copula transformation procedure.
The exact copula-transformed iDCI solution introduced represents the ideal target of this iteration, and Theorem~\ref{thm:mtd1_converges} immediately yields the following consequence.

\begin{corollary}[Approximate copula-transformed iDCI]\label{cor:approx_cc_iDCI}
Let $\{P_k^0\}_{k=0}^\infty$ be the sequence of initial probability
measures generated by~\eqref{eq:operator_iteration}, and suppose that the
associated fixed feasible set $\p$ and sequence $\{P_k^0\}$ satisfy the
hypotheses of Theorem~\ref{thm:mtd1_converges}. Then, $T_{\mathrm{iDCI}}(P_k^0)
\xrightarrow{\mathrm{TV}}
T_{\mathrm{iDCI}}(P_\infty^0)$.
Consequently, if the limiting initial probability measure coincides with
the exact copula-transformed iDCI solution, $P_\infty^0=P_{\mathrm{CT}}$, then $T_{\mathrm{iDCI}}(P_k^0)
\xrightarrow{\mathrm{TV}}
P_{\mathrm{DCI}}$.
\end{corollary}

\begin{proof}
The first statement follows directly from Theorem~\ref{thm:mtd1_converges} by identifying the sequence $\{P_k^0\}$ generated by~\eqref{eq:operator_iteration}
with the sequence of reference measures in that theorem. If $P_\infty^0=P_{\mathrm{CT}}$, then
\[
T_{\mathrm{iDCI}}(P_\infty^0)
=
T_{\mathrm{iDCI}}(P_{\mathrm{CT}})
=
P_{\mathrm{DCI}},
\]
where the final equality follows from Theorem~\ref{thm:cc_i_equals_DCI}.
\end{proof}

\subsection{Nested sequences of feasible sets via progressive information enrichment}\label{subsec:nested-feasible}

Unlike Theorem~\ref{thm:mtd1_converges}, which studies sequences of increasingly accurate reference measures, the next result considers sequences of increasingly restrictive feasible sets.
Such sequences arise naturally in a variety of settings.
Within the context of this work, one may begin with a relatively coarse GSIP involving only a subset of the component maps of a vector-valued QoI and progressively enrich the problem by incorporating additional observable components, thereby introducing additional marginal push-forward constraints.
Likewise, if prior probabilistic information regarding the parameters is available, one may progressively enforce marginal constraints on coordinate projections of the parameter space, resulting in a sequence of increasingly informative admissible sets.
More generally, one may gradually incorporate partial information about the dependence structure itself, for example by successively enforcing constraints associated with pairwise, low-dimensional, or higher-order copulas.
Finally, nested feasible sets arise naturally when employing increasingly expressive families of copula models, such as Gaussian copulas, elliptical copulas, vine copulas, or normalizing-flow-based copulas.
In each of these examples, the underlying optimization problem is solved over a nested sequence of feasible sets whose intersection defines the desired limiting problem.
The following theorem establishes convergence of the corresponding sequence of optimal solutions under general assumptions.

\begin{theorem}\label{thm:mtd2_converges}
Let $\p_{1} \supset \p_{2} \supset \dots$, where each $\p_k$ is a nonempty convex set that is closed in the total variation topology and for which the corresponding I-projection of $P^0$ exists with finite KL divergence.
Let $\p= \bigcap_{j=1}^{\infty}\p_j$ and suppose that the I-projection of $P^0$ onto $\p$ exists with finite KL divergence and that
the sequence $\{\argmin_{P \in \p_{k}}D_{\rm KL}(P \| P^{0})\}$ is relatively compact.
Then,
$
\argmin_{P \in \p_{k}}D_{\rm KL}(P \| P^{0})
\xrightarrow{\mathrm{TV}}
\argmin_{P \in \p}D_{\rm KL}(P \| P^{0})
$.
\end{theorem}

\begin{proof}
By construction, it is clear that
\begin{equation}\label{eq:limsup}
    \limsup_{k} \inf_{P \in \p^{k}} D_{\rm KL}(P || P^{0}) \le \inf_{P \in \p} D_{\rm KL}(P || P^{0}).
\end{equation}
Since the sequence
$\left\{
\arg\min_{P\in\p_k}
D_{\mathrm{KL}}(P\|P^0)
\right\}_{k\in\mathbb{N}}$
is relatively compact in total variation, there exists a subsequence $\{k_\ell\}$ and a probability measure $P^\ast$ such that
\[
\arg\min_{P\in\p_{k_\ell}}
D_{\mathrm{KL}}(P\|P^0)
\xrightarrow{\mathrm{TV}}
P^\ast.
\]
For any fixed $j$, we have $k_\ell\geq j$ for all sufficiently large
$\ell$, and therefore, by nestedness,
\[
P_{k_\ell}^\ast\in\p_{k_\ell}\subseteq\p_j.
\]
Since $\p_j$ is closed in the total variation topology and
$P_{k_\ell}^\ast\xrightarrow{\mathrm{TV}}P^\ast$, it follows that
$P^\ast\in\p_j$. Since this holds for every $j$,
\[
P^\ast\in\bigcap_{j=1}^{\infty}\p_j=\p.
\]
By lower semicontinuity of the KL divergence with respect to total variation,
\begin{equation}\label{eq:liminf}
    \liminf_{k_{l}} \inf_{P \in \p^{k_{l}}} D_{\rm KL}(P || P^{0}) = \liminf_{k_{l}}D_{\rm KL}\left(\argmin_{P \in \p^{k_{l}}}D_{\rm KL}(P | P^{0})\, \| \, P^{0}\right)  \ge  D_{\rm KL}(P^{*} || P^{0}) \ge \inf_{P \in \p} D_{\rm KL}(P || P^{0}).
\end{equation}
By~\eqref{eq:limsup} and~\eqref{eq:liminf}, we have
\begin{equation*}
    \lim_{k_{l}} \inf_{P \in \p^{k_{l}}} D_{\rm KL}(P || P^{0}) = \inf_{P \in \p} D_{\rm KL}(P || P^{0}).
\end{equation*}
By the existence and uniqueness of the minimizer of the KL divergence and relative compactness, we have
\begin{equation*}
    \lim_{k_{l}} \argmin_{P \in \p^{k_{l}}}D_{\rm KL}(P \| P^{0}) = \argmin_{P \in \p}D_{\rm KL}(P \| P^{0}).
\end{equation*}
Since the above argument can be applied to any subsequence of $\argmin_{P \in \p^{k}}D_{\rm KL}(P | P^{0})$ to extract a further subsequence with the same limit, it follows that
\begin{equation*}
    \lim_{k} \argmin_{P \in \p^{k}}D_{\rm KL}(P \| P^{0}) = \argmin_{P \in \p}D_{\rm KL}(P \| P^{0}),
\end{equation*}
which finishes the proof.
\end{proof}

Theorems~\ref{thm:mtd1_converges} and~\ref{thm:mtd2_converges} address two complementary limits.
Theorem~\ref{thm:mtd1_converges} considers convergence under a sequence of increasingly accurate reference measures with a fixed feasible set, whereas Theorem~\ref{thm:mtd2_converges} considers convergence under a nested sequence of feasible sets with a fixed reference measure.
In many applications, both mechanisms occur simultaneously where one updates the reference measure from one stage to the next while also enriching the feasible set.
The following corollary combines these two convergence mechanisms via a diagonal argument, and this progressive refinement strategy is demonstrated in the engineering example of Section~\ref{subsec:trommels}.

\begin{corollary}[Progressive refinement of reference measures and feasible sets]\label{cor:feasible_sets}
Let $\p_{1} \supset \p_{2} \supset \dots \supset \p$
where each $\p_k$ is closed in the total variation topology.
Let
$\{P_m^0\}_{m=1}^{\infty}$ be a sequence of initial probability measures
such that
\[
\eta_m
:=
\left\|
\log\left(
\frac{dP_m^0}{dP_\infty^0}
\right)
\right\|_{L^\infty(P_\infty^0)}
\longrightarrow 0.
\]
For each pair $(k,m)$ and each $k$ define, respectively
\[
Q_{k,m}
:=
\arg\min_{P\in\p_k} D_{\mathrm{KL}}(P\|P_m^0), \qquad
Q_k^\infty
:=
\arg\min_{P\in\p_k} D_{\mathrm{KL}}(P\|P_\infty^0),
\]
and further define $
Q_\infty
:=
\arg\min_{P\in\p} D_{\mathrm{KL}}(P\|P_\infty^0).
$
Assume the hypotheses of Theorems~\ref{thm:mtd1_converges} and~\ref{thm:mtd2_converges} hold. Then there exists an increasing sequence of indices $m(k)\to\infty$ such that $Q_{k,m(k)} \xrightarrow{\mathrm{TV}} Q_\infty$.
\end{corollary}

\begin{proof}
For a fixed $k$,  Theorem~\ref{thm:mtd1_converges} applied to the sequence $\{P_m^0\}_{m=1}^\infty$ implies $Q_{k,m} \xrightarrow{\mathrm{TV}} Q_k^\infty$ as $m\to\infty$.
Hence, for each $k$ we may choose an index $m(k)$, with $m(k)$ strictly increasing, such that
$d_{\mathrm{TV}}(Q_{k,m(k)},Q_k^\infty) \le \frac{1}{k}$.
Next, applying Theorem~\ref{thm:mtd2_converges} to the nested feasible sets $\{\p_k\}$ implies $Q_k^\infty \xrightarrow{\mathrm{TV}} Q_\infty$ as $k\to\infty$.
Therefore, by the triangle inequality,
\[
d_{\mathrm{TV}}(Q_{k,m(k)},Q_\infty)
\le
d_{\mathrm{TV}}(Q_{k,m(k)},Q_k^\infty)
+
d_{\mathrm{TV}}(Q_k^\infty,Q_\infty)
\le
\frac{1}{k}
+
d_{\mathrm{TV}}(Q_k^\infty,Q_\infty),
\]
and the right-hand side tends to zero as $k\to\infty$.
Thus, $Q_{k,m(k)} \xrightarrow{\mathrm{TV}} Q_\infty$,
which proves the claim.
\end{proof}

\section{Computational Approximations}\label{sec:comp_approx}

Recall that the iDCI solution to the GSIP is defined as the limit of a sequence of local solutions to subproblems given by~\eqref{sub_min}.
Since each of these subproblems is simply defined as the DCI solution to a SIP with a given initial and observed distribution, we first describe, at a high-level, the manner in which a DCI solution can be approximated and the predictability assumption numerically verified.

\subsection{Numerical approximation of DCI}

First, we rewrite~\eqref{eq:updated_density} as
\begin{equation}\label{eq:updated_density_r}
    \updatedens(\lambda) = \initdens(\lambda) r(\lambda), \ \text{ where } \ r(\lambda) := \frac{\obsdens(\phi(\lambda))}{\predictdens(\phi(\lambda))}.
\end{equation}
For a given set of $N$ independent identically distributed (iid) samples $\{\lambda^{(i)}\}\sim\initdens$, we refer to the corresponding $\{r(\lambda^{(i)})\}$ values as simply the ``$r$-values'' that re-weight this iid set of samples so that the weighted sample set can be interpreted as being drawn from $\updatedens$.
We can, for instance, perform rejection sampling (as in \cite{BJW18a}) on such a sample set to generate an iid set of samples from $\updatedens$.
Note that if the predictability assumption holds, then $\updatedens$ is indeed a density and it follows that
\begin{equation}
     \mathbb{E}_\text{init}(r(\lambda)) = \int_\pspace \updatedens(\lambda) d\mu_\pspace = 1 \ \Rightarrow \ \frac{1}{N}\sum_i r(\lambda^{(i)}) \approx 1.
\end{equation}
This naturally leads to a computational diagnostic where, for a given set of parameter samples and associated $r$-values, we check the validity of the predictability assumption by verifying that
\begin{equation}\label{eq:diagnostic}
    \left\vert \frac{1}{N}\sum_i r(\lambda^{(i)}) - 1 \right\vert < tol_r
\end{equation}
for a specific tolerance $tol_r$.
This diagnostic is utilized within each iterated QoI subspace of Algorithm~\ref{alg:main}.

\subsection{Numerical approximation of iDCI}

\begin{algorithm}\caption{Iterative DCI (iDCI)}\label{alg:main}
\begin{algorithmic}
    \REQUIRE Predicted QoI samples, initialized $r$-values, observed densities, tolerances and $\max_\text{epoch}.$
    \FOR {epoch = $1,\ldots, \max_{\text{epoch}}$}
        \FOR {each QoI subspace}
        \STATE {Normalize current $r$-values to have unit sample mean.}
            \STATE {Estimate predicted density on subspace with QoI samples and $r$-values.}
            \STATE {Multiplicatively update the $r$-values using DCI.}
            \IF {diagnostic of $r$-values within tolerance}
        	\STATE {Continue to next subspace.}
            \ELSE
                \STATE {Exit both loops.}
            \ENDIF
        \ENDFOR
        \IF {absolute or relative KL divergences of QoI marginals within tolerances}
        	\STATE {Exit loop.}
            \ELSE
                \STATE {Continue to next epoch.}
        \ENDIF
    \ENDFOR
    \ENSURE Final iDCI $r$-values.
\end{algorithmic}
\end{algorithm}

Assume that $\{r(\lambda^{(i)})\}$ has been computed as above for both some particular set of iid samples from $\initmeas$, QoI map $\phi$, and observed probability measure associated with the output of $\phi$.
Observe that applying a different QoI map $\tilde{\phi}$ to generate $\{\tilde{\phi}(\lambda^{(i)})\}$ and then computing the weighted KDE with the prior $r$-values will generate an approximation of the push-forward density defined by $\updatedens$ and $\tilde{\phi}$.
In other words, neither explicit approximation of $\updatedens$ nor direct generation of iid samples from $\updatedens$ is necessary to construct an approximation to its push-forward density associated with a different QoI map.
This is a critical observation utilized in Algorithm~\ref{alg:main} (taken from \cite{JBW+26}) to generate the sequence of local solutions to the subproblems that approximates the iDCI solution to the GSIP.

As mentioned in~\cite{JBW+26}, the algorithm avoids an overly prescriptive approach to estimating densities and applying DCI at each iteration.
In the numerical results of Section~\ref{sec:numerics}, we provide more specific computational details for each example related to the use of parametric or KDE estimates of densities.

\subsection{Numerical approximation of copula-transformed iDCI}\label{subsec:num_approx_copulas}

\begin{algorithm}
\caption{Copula-Transformed iDCI (CT+iDCI)}
\label{alg:ctidci}
\begin{algorithmic}
    \REQUIRE Predicted QoI samples, initial $r$-values, observed QoI samples,
    copula model, tolerances and $\max_\text{iter}$.
    \STATE {Estimate the observed copula model from the observed QoI samples.}
    \FOR {iteration = $1,\ldots,\max_{\text{iter}}$}
        \STATE {Apply Algorithm~\ref{alg:main} to obtain the final iDCI
        $r$-values.}
        \STATE {Estimate the predicted copula model using the final iDCI $r$-values.}
        \STATE {Update the $r$-values by the ratio of the estimated
        observed and predicted copula densities in the common copula coordinate system.}
        \STATE {Use the updated $r$-values as the initial $r$-values for
        the next iteration.}
        \IF {CT+iDCI convergence criteria are satisfied}
            \STATE {Exit loop.}
        \ELSE
            \STATE {Continue to next iteration.}
        \ENDIF
    \ENDFOR
    \ENSURE Final CT+iDCI $r$-values.
\end{algorithmic}
\end{algorithm}

Algorithm~\ref{alg:ctidci} summarizes the sequential copula-transformed iDCI procedure given by~\eqref{eq:operator_iteration}, with the copula-estimation step providing the only component not already specified by Algorithm~\ref{alg:main}.
The copula transformation of the iDCI solution requires approximations of the copulas associated with the observed distribution and the {\em joint} push-forward of the iDCI solution.
In principle, these copulas may be estimated using any suitable parametric or non-parametric approach as mentioned in Section~\ref{sec:iDCI_sequences}.
An important computational consequence of Theorem~\ref{thm:copula+iDCI} is that, after iDCI has enforced the marginal push-forward constraints, the remaining dependence structure may be approximated independently of the marginal density models.
In the numerical examples of Section~\ref{sec:numerics}, we exploit this flexibility by using Gaussian copulas.

For the Gaussian copula approximations used in this work, the dependence structure is characterized by an estimated correlation matrix.
For the observed distribution, this correlation is estimated directly from the observed QoI samples.
For the predicted distribution associated with an
iDCI solution, the corresponding correlation is estimated from the predicted QoI samples using the final iDCI $r$-values as probability weights.
The observed marginal CDFs are then used to map the QoI samples to the common copula coordinate system identified in Theorem~\ref{thm:copula+iDCI}.
In the Gaussian examples of Section~\ref{sec:numerics}, these CDFs are available from the fitted parametric Gaussian distributions, whereas in the non-Gaussian examples they are obtained by integrating the corresponding KDE approximations.
In either case, the resulting probability-integral transforms are mapped to standard normal coordinates and used to evaluate Gaussian copula densities parameterized by the estimated observed and predicted correlations.
The approximate copula transformation is then implemented by multiplying the existing iDCI $r$-values by the ratio of these estimated copula densities, so that no explicit construction of the transformed density on parameter space is required.
The interested reader is referred to Section~\ref{sec:supplementary} for information on obtaining and executing the code used to construct the copula approximations and transformations presented below.

\section{Numerical Examples}\label{sec:numerics}

This section presents three numerical examples illustrating the principal theoretical developments of Sections~\ref{sec:copulas} and~\ref{sec:iDCI_sequences}.
The first example demonstrates how the geometry induced by the QoI map significantly impacts the importance of the copula transformation.
The second illustrates an adaptive reference-measure refinement strategy motivated by Theorem~\ref{thm:mtd1_converges} and Corollary~\ref{cor:approx_cc_iDCI} for improving density and copula estimation under a fixed computational budget.
The final example considers a computational mechanics application involving asynchronously acquired experimental data and illustrates a progressive feasible-set enrichment strategy motivated by Theorem~\ref{thm:mtd2_converges} and Corollary~\ref{cor:feasible_sets}, while illustrating the flexibility of iDCI with heterogeneous vector-valued QoI maps.

The latter two examples are not intended to verify the sufficient conditions or asymptotic convergence results of Section~\ref{sec:iDCI_sequences}.
Rather, they demonstrate finite-sample computational strategies motivated by these results: adaptive refinement of the reference measure and progressive enrichment of the feasible set, respectively.

\subsection{Dependence transformation in Linear Inverse Problems}\label{subsec:copulas_linear_example}

To isolate the effect of the induced geometry from the QoI map, we first consider Gaussian initial and data-generating distributions together with linear QoI maps so that all probability distributions may be represented parametrically.
The results subsequently show that the only significant source of discrepancy between the iDCI and DCI solutions arises from the dependence structure induced by the QoI map rather than from numerical approximation errors in density estimation.
Then, to demonstrate that parametric copula estimates can significantly improve non-parametric estimates of iDCI solutions, we  consider non-Gaussian initial and data-generating distributions and utilize KDEs for estimating the associated predicted and observed distributions.

\subsubsection{Gaussian Distributions and the Impact of QoI Geometries}

Consider the following three QoI maps:
\begin{align}
    Q_1(\lambda) &= 2\lambda_1 + \lambda_2, \\
    Q_2(\lambda) &= 2.5\lambda_1 + 0.5\lambda_2, \\
    Q_3(\lambda) &= \lambda_1-\lambda_2.
\end{align}
We assume that $Q_1(\lambda)$ is currently being measured and investigate the impact of choosing $Q_2$ or $Q_3$ as an additional measurement to construct iDCI solutions with and without copula transformations.
The pair $(Q_1,Q_2)$ induces a highly skewed inverse geometry because the level sets of the two component maps are nearly parallel, whereas $(Q_1,Q_3)$ produces a substantially better-conditioned inverse geometry with level sets that intersect at a much larger angle.
We therefore refer to the pair $(Q_1, Q_2)$ as defining the ``skewed QoI map'' and $(Q_1, Q_3)$ as defining the ``non-skewed QoI map'' to help clarify the discussion.
We set
\[
\initdens\sim N(\mathbf{0}, I_{2\times 2}) \quad \text{ and } \quad
\pi_{\rm DG}\sim N(\mu, \Sigma) \quad \text{ where } \quad \mu=\mat{0.5 \\ 0.5} \quad \text{ and } \quad \Sigma = \mat{0.5 & 0.4 \\ 0.4 & 0.5}
\]
In all results, we generate $10^4$ iid samples from the initial distribution and $10^3$ iid samples from the DG distribution.
Unbiased estimates of the means and covariances are then utilized to compute the parametric estimates for each of the Gaussian distributions involved in this example.
All KL divergences reported for this example are computed using the exact formula for the KL divergence between two Gaussian distributions with specified means and covariances.

\begin{figure}
    \centering
    \includegraphics[width=0.23\linewidth]{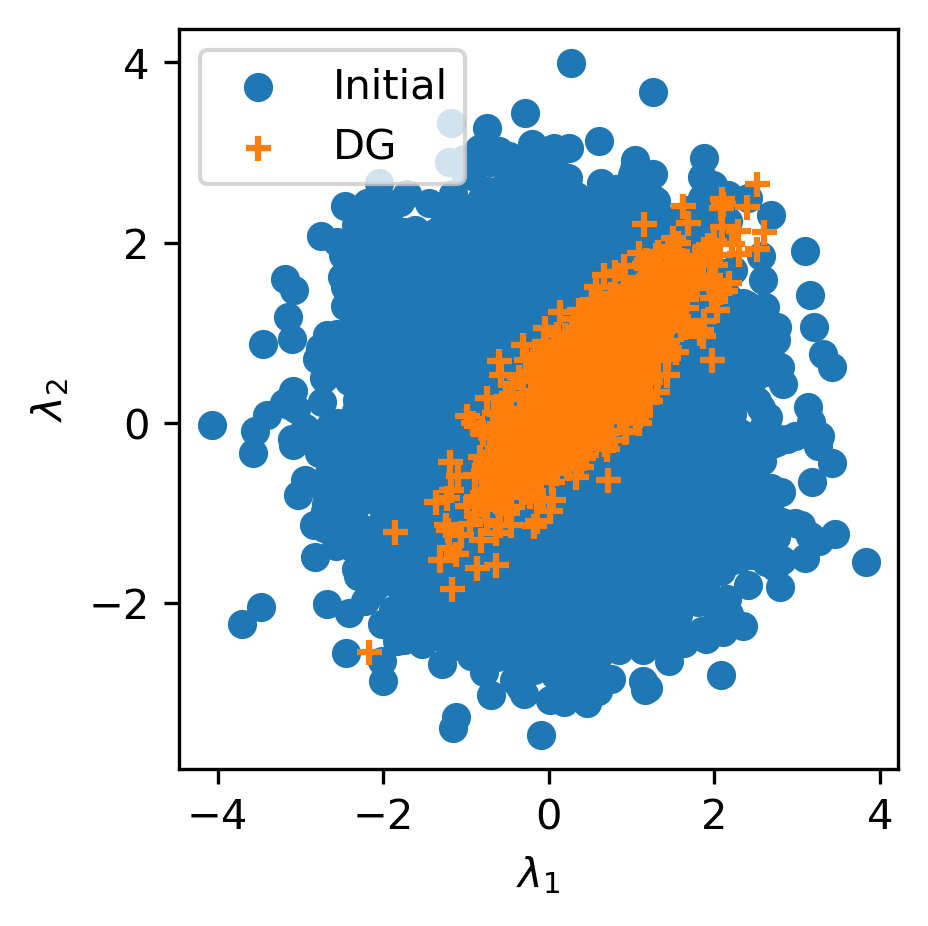}
    \includegraphics[width=0.23\linewidth]{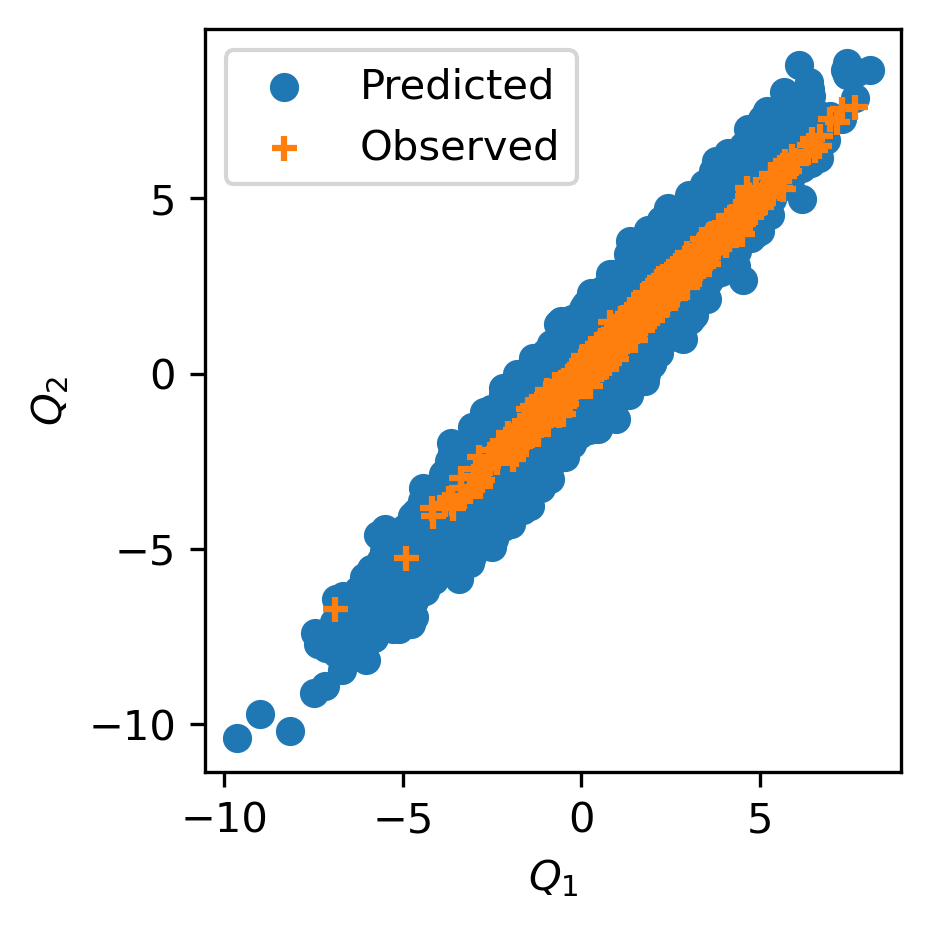}
    \includegraphics[width=0.23\linewidth]{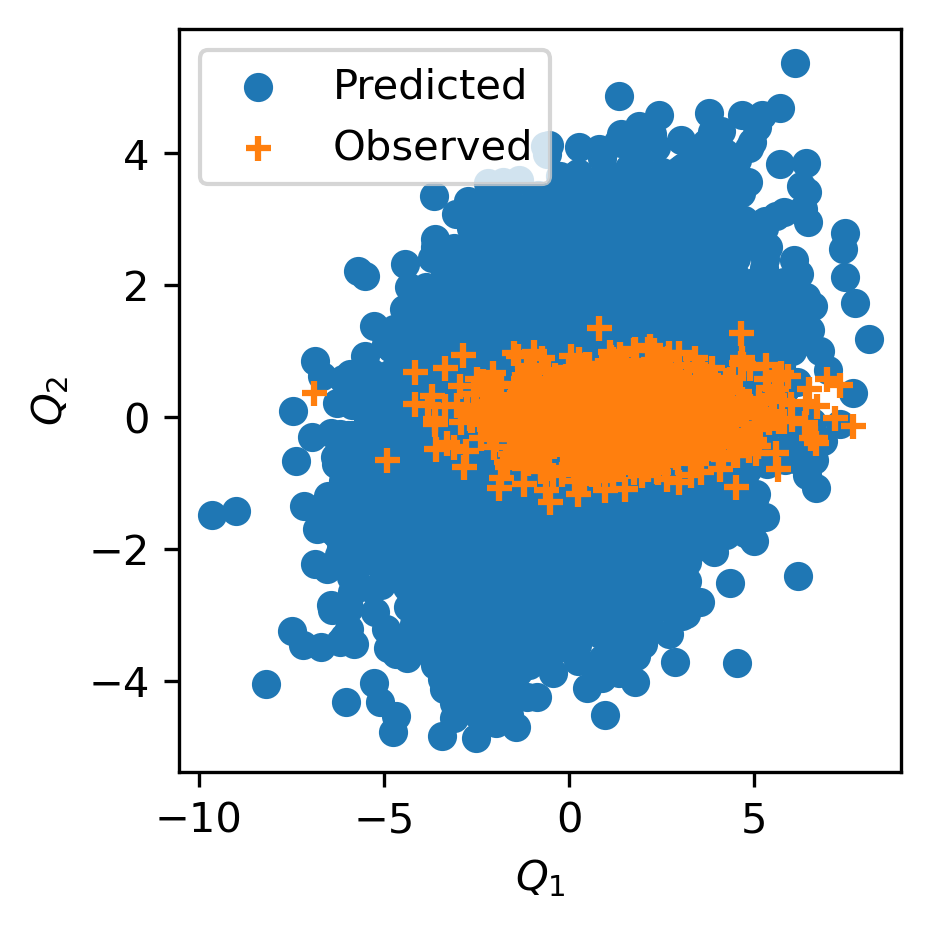}
    \caption{Parameter samples (left) and corresponding observable samples (center/right) for the skewed and non-skewed linear QoI maps.}
    \label{fig:ex1-Gaussian-param-space-dataspaces}
\end{figure}

Figure~\ref{fig:ex1-Gaussian-param-space-dataspaces} illustrates the parameter and QoI samples.
The left plot shows samples from the initial and data-generating parameter distributions, while the center and right plots compare the predicted and observed QoI samples associated with the skewed and non-skewed QoI maps, respectively.
Although both QoI maps are linear, the skewed map produces observable distributions whose dominant direction is nearly aligned with the reduced-variance direction of the data-generating distribution.
This geometric alignment is the primary source of the substantial dependence transformation observed below.
Specifically, since the iDCI algorithm enforces only the marginal push-forward constraints, any remaining discrepancy between the joint predicted and observed distributions is attributable to the copula, as established in Theorem~\ref{thm:KL+copula+iDCI}.
Consequently, the skewed QoI map provides a setting in which a substantial copula transformation is anticipated.

\begin{figure}
    \centering
    \includegraphics[width=0.75\linewidth]{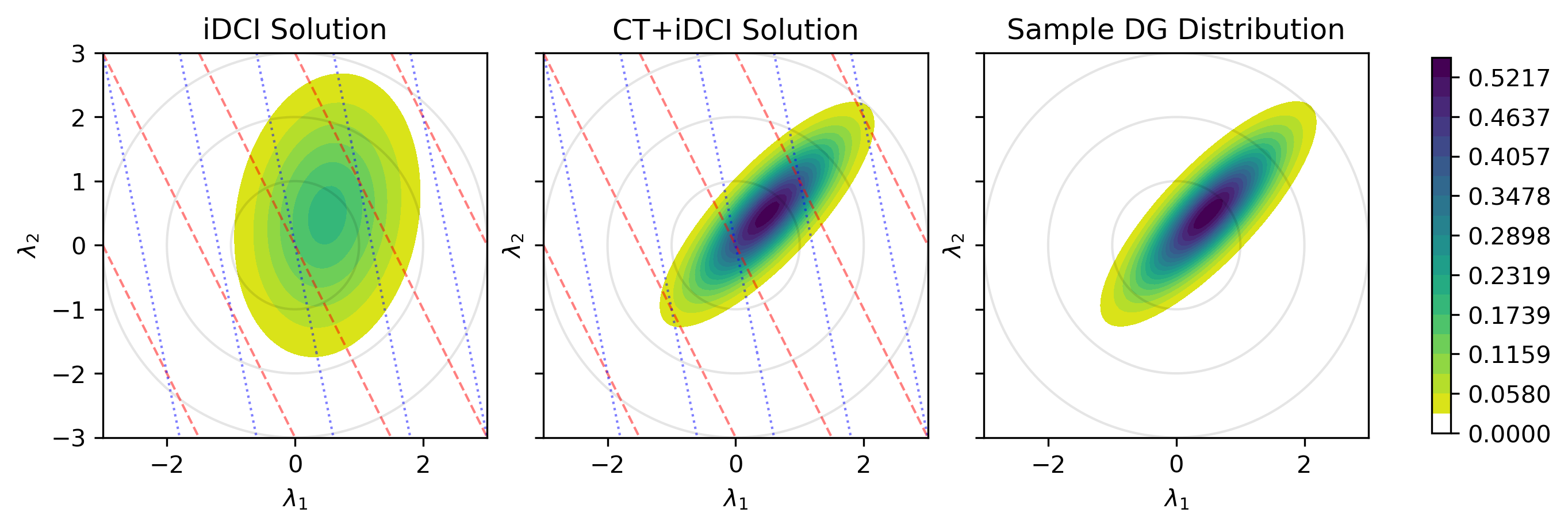}
    \caption{Parameter-space densities for the Gaussian example with the skewed QoI map. Left: iDCI solution. Center: CT+iDCI solution obtained using an exact Gaussian copula transformation. Right: Data-generating (DG) distribution. Gray circles denote Mahalanobis-distance contours of the standard normal initial distribution, while dashed and dotted lines indicate level sets of the component QoI maps.}
    \label{fig:ex1-Gaussian-CT-iDCI-skewed-QoI.png}
\end{figure}

In the parameter density comparison plots of Figures~\ref{fig:ex1-Gaussian-CT-iDCI-skewed-QoI.png} and~\ref{fig:ex1-Gaussian-iDCI-non-skewed-QoI.png}, the light gray circles correspond to isodensity contours of the standard bivariate normal distribution (i.e., the initial distribution) with Mahalanobis distances $r=1, 2, $ and $3$.
These contours provide geometric benchmarks for comparison with both the transformed distributions (both iDCI and CT+iDCI solutions) as well as the DG distribution.
Additionally, in the iDCI and CT+iDCI solution plots, we show five contour lines associated with each of the components in the QoI maps to provide geometric intuition regarding the impact these skewed and non-skewed structures have on the iDCI solution prior to the copula transformation.

\begin{figure}
    \centering
    \includegraphics[width=0.75\linewidth]{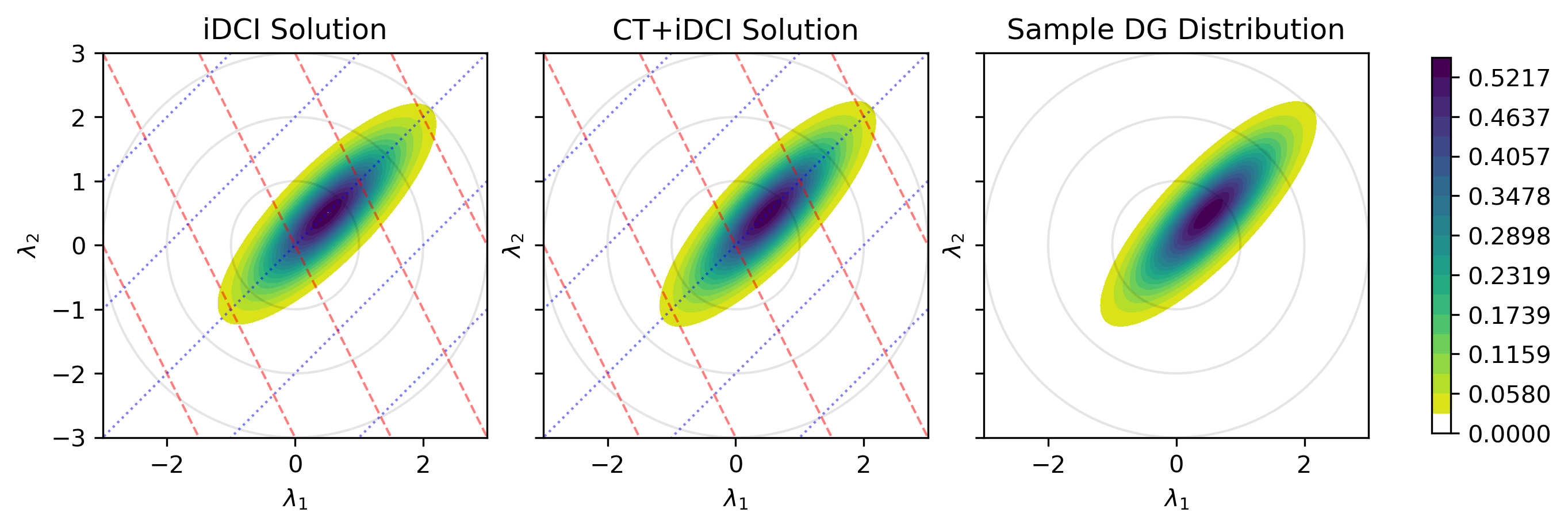}
    \caption{Parameter-space densities for the Gaussian example with the non-skewed QoI map. Left: iDCI solution. Center: CT+iDCI solution. Right: Data-generating (DG) distribution. Gray circles denote Mahalanobis-distance contours of the standard normal initial distribution, while dashed and dotted lines indicate level sets of the component QoI maps.}
    \label{fig:ex1-Gaussian-iDCI-non-skewed-QoI.png}
\end{figure}

It is evident from Figure~\ref{fig:ex1-Gaussian-CT-iDCI-skewed-QoI.png} that the iDCI solution associated with the skewed QoI map captures the marginal behavior of the data-generating distribution but fails to reproduce its dependence structure, leading to a KL divergence of approximately $\expnumber{6.433}{-1}$.
Applying the copula transformation produces a distribution that is visually indistinguishable from the data-generating distribution and significantly reduces the KL divergence to approximately $\expnumber{4.929}{-4}$, numerically illustrating Theorem~\ref{thm:cc_i_equals_DCI} in the Gaussian setting.

For the non-skewed QoI map results shown in Figure~\ref{fig:ex1-Gaussian-iDCI-non-skewed-QoI.png}, the geometry induced by the QoI map already produces an iDCI solution that closely matches the data-generating distribution, with a KL divergence of only $\expnumber{1.255}{-3}$.
Consequently, the copula transformation produces only a modest additional improvement, reducing the KL divergence to $\expnumber{1.058}{-5}$.
These two examples demonstrate that the importance of the copula transformation is governed primarily by the geometry induced by the QoI map rather than by the marginal distributions themselves.

\subsubsection{Non-Gaussian Distributions with a Gaussian Copula Approximation}

\begin{figure}
    \centering
    \includegraphics[width=0.23\linewidth]{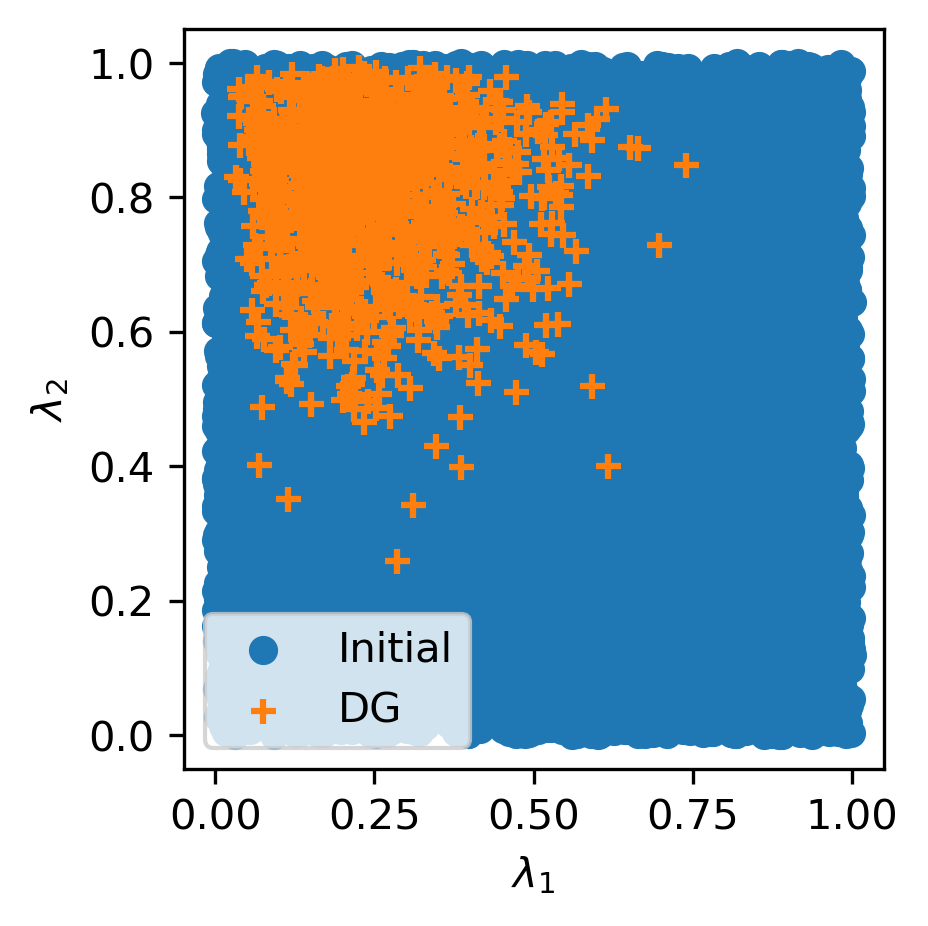}
    \includegraphics[width=0.23\linewidth]{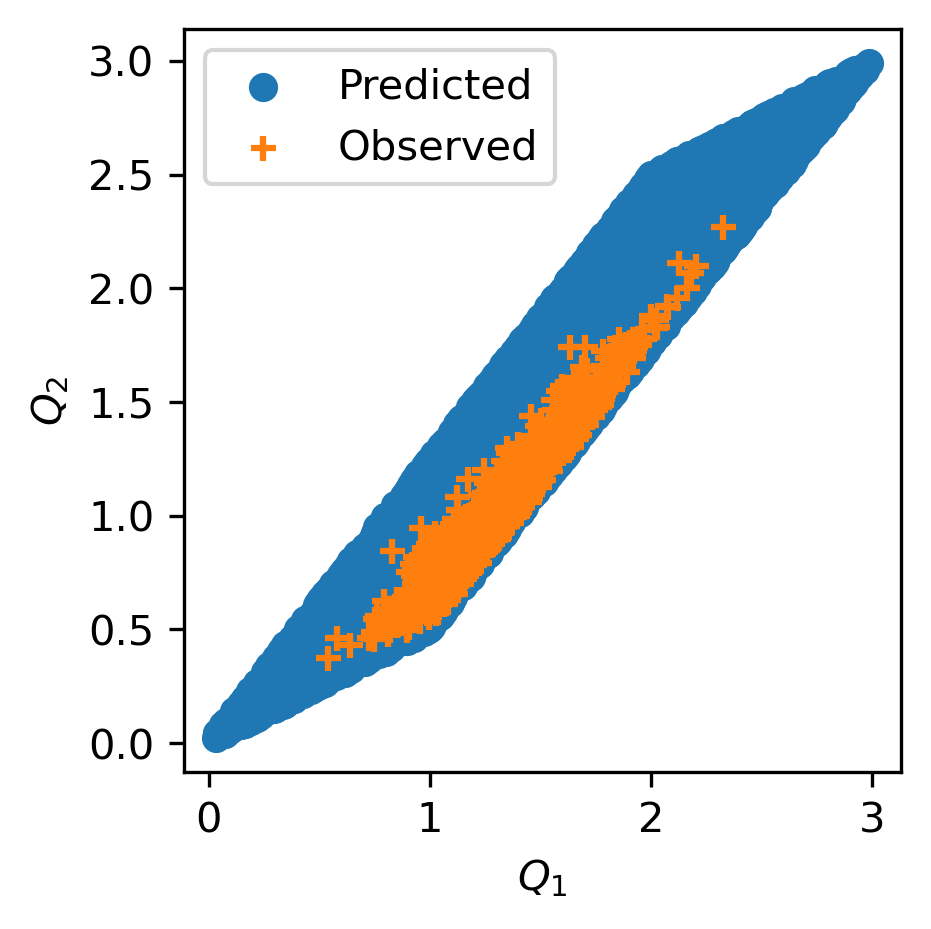}
    \includegraphics[width=0.23\linewidth]{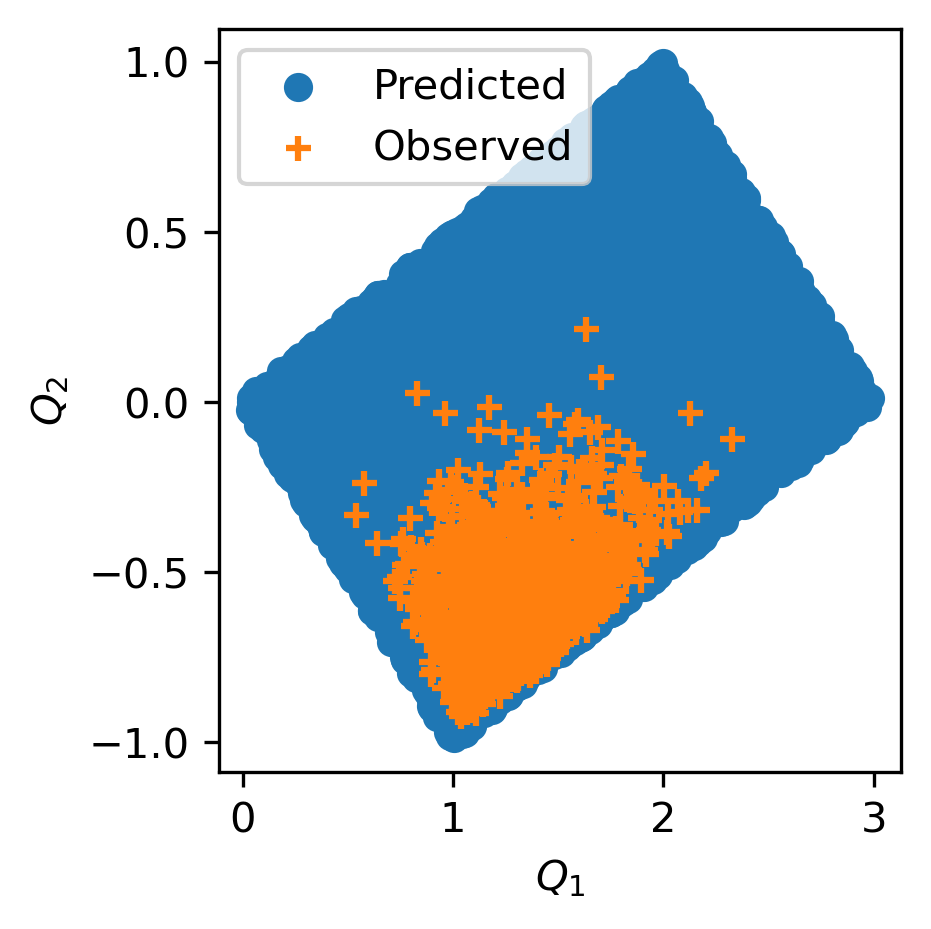}
    \caption{Parameter and observable samples for the non-Gaussian example. Left: Initial parameter samples (blue) and data-generating (DG) samples (orange). Center: Predicted and observed samples for the skewed QoI map. Right: Predicted and observed samples for the non-skewed QoI map.}
    \label{fig:ex1-Betas-param-space-dataspaces}
\end{figure}

The previous example intentionally removed nearly all numerical approximation error by restricting attention to Gaussian distributions represented through parametric estimates.
We now consider a uniform initial distribution, a non-Gaussian data-generating distribution constructed from independent Beta distributions, weighted kernel density estimates used for non-parametric estimation of the marginal predicted densities computed in the iDCI algorithm, and a Gaussian approximation of the copula.
In other words, both the marginal predicted densities and the copula are estimated numerically.
This allows us to assess the robustness of the proposed copula transformation methodology under more realistic computational conditions.
Since the non-parametric copula is estimated with a Gaussian, we expect that applying it to transform the iDCI solution will result in a probability measure that fails to satisfy the push-forward constraints, which suggests applying the iDCI algorithm once more, consistent with the iterative framework underlying Corollary~\ref{cor:approx_cc_iDCI} that is explored in more depth in Section~\ref{subsec:adaptive_measures}.
With this in mind, we therefore apply the iDCI algorithm to the copula-transformed solution, and refer to that output as the CT+iDCI solution.

Since the iDCI and CT+iDCI solutions are compared to a kernel density estimate of the data-generating distribution, it is useful to place the reported KL divergences in the context of the sampling error associated with such estimates of this data-generating distribution.
Using 50 independent sets of $\expnumber{1}{3}$ samples drawn from the exact data-generating distribution to generate other KDE estimates, the average KL divergence between these and the one used in the experiments was $\expnumber{3.420}{-2}$, with a minimum of $\expnumber{2.058}{-2}$ and a maximum of $\expnumber{5.562}{-2}$.
Thus, we consider any iDCI or CT+iDCI solution with a KL divergence less than approximately $\expnumber{5}{-2}$ from the sample data-generating distribution as being on par with a separate direct approximation of the data-generating distribution.

\begin{figure}
    \centering
    \includegraphics[width=0.75\linewidth]{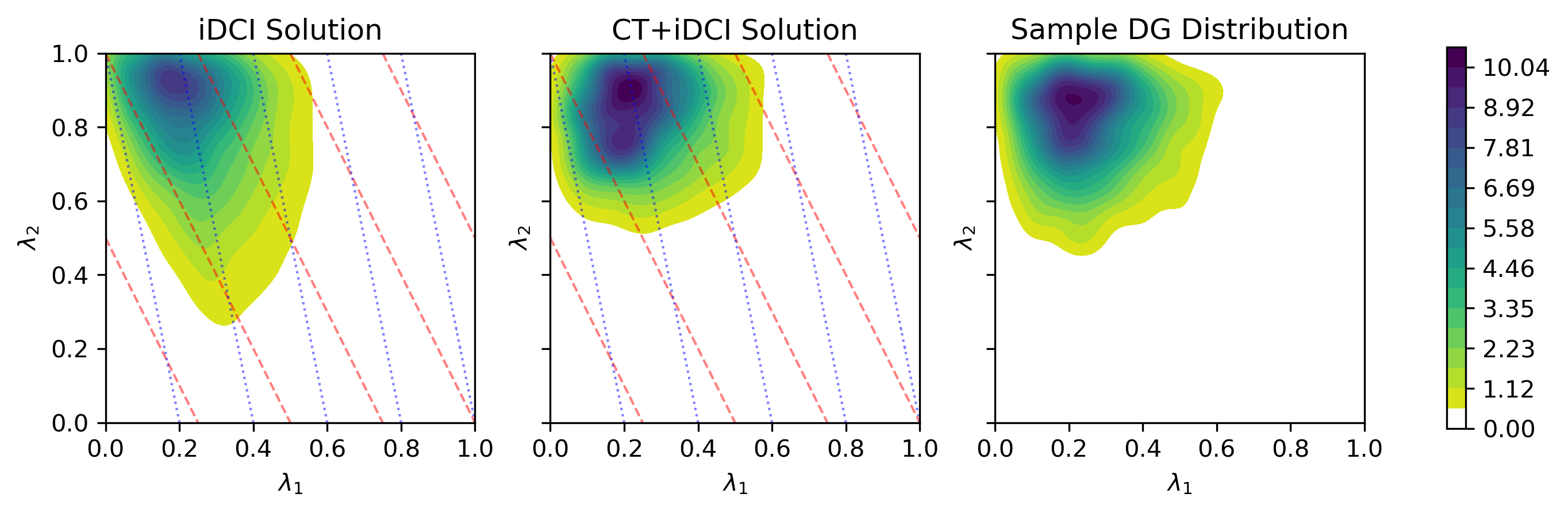}
    \caption{Parameter-space densities for the non-Gaussian example with the skewed QoI map. Left: iDCI solution estimated using weighted kernel density estimation. Center: CT+iDCI solution obtained using an approximate Gaussian copula transformation. Right: Kernel density estimate of the sampled data-generating distribution. Dashed and dotted lines indicate level sets of the component QoI maps.}
    \label{fig:ex1-Betas-CT-iDCI-skewed-QoI.png}
\end{figure}

Figure~\ref{fig:ex1-Betas-param-space-dataspaces} illustrates the parameter and observable samples associated with the non-Gaussian example.
Although the observable distributions differ substantially from the Gaussian case, the same distinction between the skewed and non-skewed QoI geometries remains evident.

The parameter-space densities obtained using the skewed QoI map are shown in Figure~\ref{fig:ex1-Betas-CT-iDCI-skewed-QoI.png}.
The iDCI solution exhibits a KL divergence of approximately $\expnumber{1.300}{-1}$ relative to the sample data-generating distribution, whereas the Gaussian copula transformation reduces this value to $\expnumber{3.663}{-2}$.
Although the copula model is now only an approximation, the resulting CT+iDCI solution achieves an accuracy comparable to the sampling variability associated with directly estimating the data-generating distribution at the same sample size.

\begin{figure}
    \centering
    \includegraphics[width=0.75\linewidth]{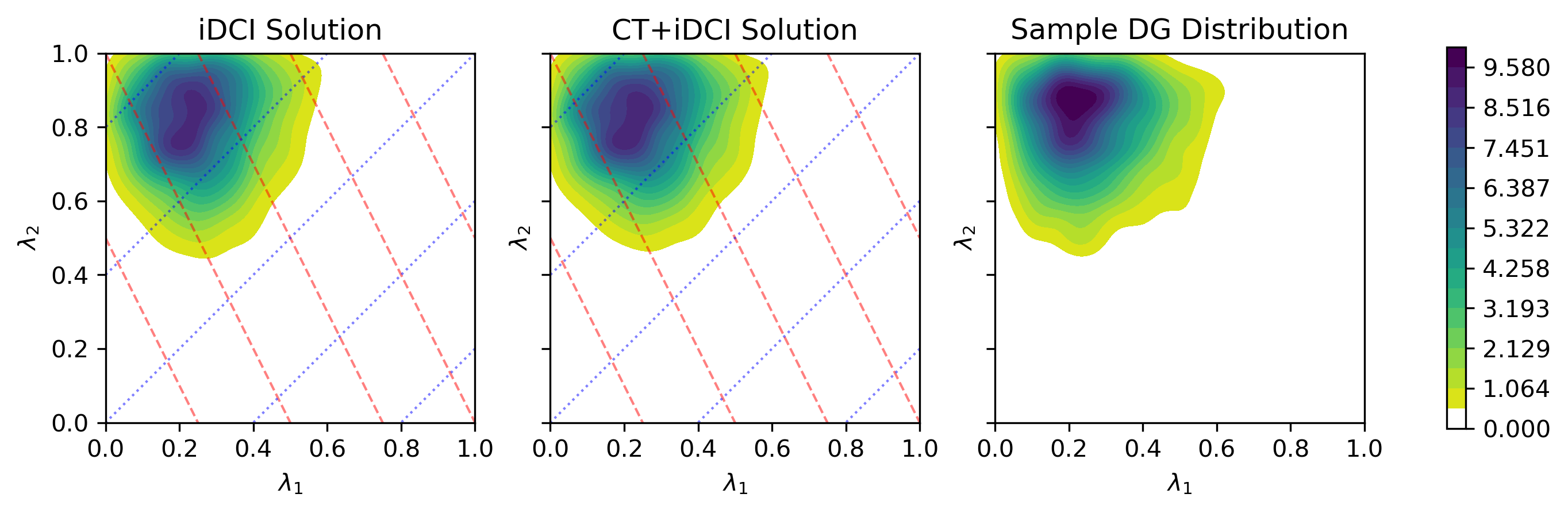}
    \caption{Parameter-space densities for the non-Gaussian example with the non-skewed QoI map. Left: iDCI solution estimated using weighted kernel density estimation. Center: CT+iDCI solution obtained using an approximate Gaussian copula transformation. Right: Kernel density estimate of the sampled data-generating distribution. Dashed and dotted lines indicate level sets of the component QoI maps.}
    \label{fig:ex1-Betas-CT-iDCI-nonskewed-QoI.png}
\end{figure}

The corresponding results for the non-skewed QoI map are shown in Figure~\ref{fig:ex1-Betas-CT-iDCI-nonskewed-QoI.png}.
As in the Gaussian example, the iDCI solution already provides an excellent approximation to the data-generating distribution (with a KL divergence of approximately $\expnumber{2.364}{-2}$ from the sample data-generating distribution), and the Gaussian copula transformation produces only a modest additional improvement that reduces the KL divergence to $\expnumber{2.084}{-2}$.
Together with the Gaussian examples, these results indicate that the geometry induced by the QoI map plays a significantly larger role in determining the magnitude of the copula transformation than the particular choice of marginal density model.

\subsubsection{Summary}

\begin{table}[ht]
\centering
\caption{Summary of the KL divergences from the sample data-generating distribution for the four linear inverse problems considered in Section~\ref{subsec:copulas_linear_example}.
The Gaussian examples utilize parametric density approximations, whereas the Beta examples employ weighted kernel density estimates together with a Gaussian copula approximation.
For the Beta examples, the final column reports the average KL divergence between the kernel density estimates of 50 separate sample data-generating distributions and the one utilized in the construction of the iDCI and CT+iDCI solutions.}
\label{tab:linear_examples}

\begin{tabular}{l|c|c|c}
\toprule
Example &
iDCI &
CT+iDCI &
DG Reference \\
\midrule
Gaussian (Skewed)
&
$\expnumber{6.433}{-1}$
&
$\expnumber{4.929}{-4}$
&
-- \\
\hline
Gaussian (Non-skewed)
&
$\expnumber{1.255}{-3}$
&
$\expnumber{1.058}{-5}$
&
-- \\
\hline
Beta (Skewed)
&
$\expnumber{1.300}{-1}$
&
$\expnumber{3.663}{-2}$
&
$\expnumber{3.420}{-2}$ \\
\hline
Beta (Non-skewed)
&
$\expnumber{2.364}{-2}$
&
$\expnumber{2.084}{-2}$
&
$\expnumber{3.420}{-2}$ \\

\bottomrule
\end{tabular}
\end{table}

Table~\ref{tab:linear_examples} summarizes the quantitative results for all four linear inverse problems considered in this section.
The takeaway is that the effectiveness of the copula transformation is governed principally by the geometry induced by the QoI map rather than by the choice of marginal distributions. In both the Gaussian and non-Gaussian settings, the skewed QoI map produces substantially larger improvements from the copula transformation than the corresponding non-skewed QoI map.

The summary presented in Table~\ref{tab:linear_examples} also provides an important computational perspective.
In the Gaussian examples, the nearly exact parametric representations remove almost all approximation error, allowing the observed improvements to be attributed solely to the dependence transformation as anticipated by the theory developed in Section~\ref{sec:copulas}.
By contrast, the estimated copula transformation to the non-Gaussian example achieves an accuracy that is comparable to the sampling variability associated with directly estimating the data-generating distribution at the same sample size.
This observation suggests that, beyond improving the dependence transformation, further reductions in the KL divergence will require more accurate estimates of the predicted push-forward densities and copulas rather than simply a different copula model.

\subsection{Adaptive Reference-Measure Refinement}\label{subsec:adaptive_measures}

In practical applications, the computational cost of model evaluations may severely limit the number of predicted samples available, which will impact the accuracy of the push-forward densities as well as the copulas and thus the iDCI and CT+iDCI solutions.
The convergence results of Theorem~\ref{thm:mtd1_converges} motivate an adaptive strategy to address this accuracy issue when computational constraints restrict sample sizes.
Rather than repeatedly sampling from the original initial distribution, which can be quite broad and result in many samples being ``down-weighted,'' we generate new samples from the reference measure that is recursively updated using the previous CT+iDCI solution.
With the notation of equation~\eqref{eq:operator_iteration}, let $k$ denote the current stage of the adaptation and $P_{k+1}^0
=
\widehat{\mathcal T}_{\rm CT}
\!\left(
\mathcal T_{\rm iDCI}(P_k^0)
\right)$ denote the new reference measure from which samples are generated.

To illustrate this strategy, we revisit the skewed linear QoI map introduced in Section~\ref{subsec:copulas_linear_example} together with the same standard normal initial distribution and data-generating distribution.
However, unlike the previous example, only $100$ parameter samples are generated at each stage.
Beginning with the standard normal reference measure, the corresponding CT+iDCI solution is computed and subsequently approximated by a multivariate Gaussian distribution.
This Gaussian approximation then serves as the reference measure from which the next set of $100$ parameter samples is generated.
The process is repeated for three adaptive refinement stages, thereby maintaining a fixed computational budget per stage while progressively concentrating samples in regions of parameter space that are most relevant to the inverse solution.

Table~\ref{tab:adaptive_sampling} summarizes the KL divergence between the data-generating distribution and both the iDCI and CT+iDCI solutions obtained at each adaptive refinement stage.
Collectively, these results demonstrate that recursively updating the reference measure rapidly improves the accuracy of both the predicted push-forward densities and the estimated copula while maintaining a fixed number of model evaluations per stage.

\begin{table}[ht]
\centering
\caption{Adaptive refinement of the reference measure using a fixed sample size of $100$ per stage. Beginning from a standard normal reference measure, the final CT+iDCI solution obtained at each stage is used as the reference measure for the subsequent stage. The table reports the KL divergence between the data-generating (DG) distribution and the corresponding iDCI and final CT+iDCI solutions}
\label{tab:adaptive_sampling}

\begin{tabular}{c|c|c}
\toprule
Adaptive Stage: $k$ &
$D_{\rm KL}(P_{\rm DG} \| \mathcal{T}_{\rm iDCI}(P_k^0))$ &
$D_{\rm KL}(P_{\rm DG} \| \widehat{\cal T}_{\rm CT}\left(\mathcal{T}_{\rm iDCI}(P_k^0)\right))$  \\
\midrule
 &
$\expnumber{5.409}{-1}$ &
$\expnumber{1.143}{-2}$  \\
\hline
 &
$\expnumber{3.980}{-2}$ &
$\expnumber{7.153}{-3}$ \\
\hline
 &
$\expnumber{9.300}{-4}$ &
$\expnumber{1.094}{-3}$\\
\bottomrule
\end{tabular}
\end{table}

Beginning from a standard normal reference measure, successive adaptive refinement stages rapidly improve the quality of both the iDCI and CT+iDCI approximations while maintaining a fixed computational budget of $100$ samples per stage.
By the third stage, both approximations have converged to within approximately $10^{-3}$ KL divergence of the data-generating distribution, demonstrating that recursively updating the reference measure provides an effective adaptive sampling strategy for accurately estimating both the predicted push-forward densities and the associated copula.

\subsection{Enriching Feasible Sets with Asynchronous Experiments}\label{subsec:trommels}

\begin{figure}[htbp]
\centering
{\label{fig:trommelA}\includegraphics[width=0.4\textwidth]{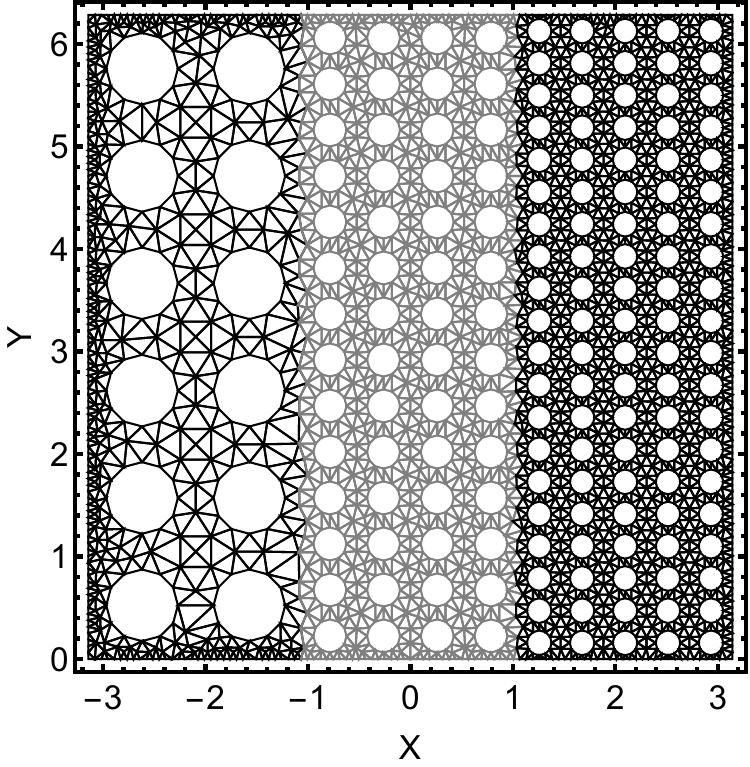}}
\quad
{\label{fig:trommelB}\includegraphics[width=0.4\textwidth]{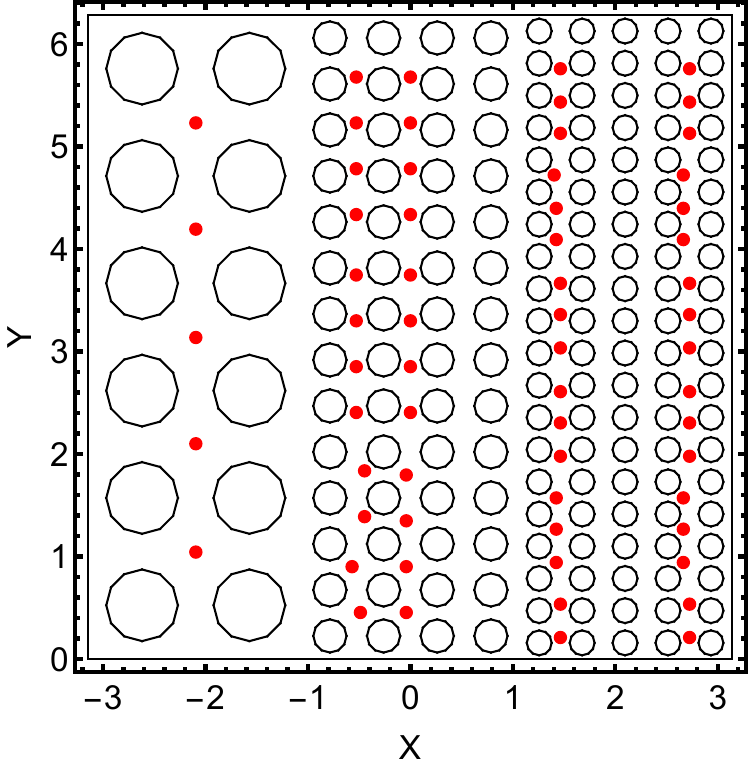}}
\caption{Trommel configuration, hole coverage 35\
The boundaries at $x = \pm \pi$ are clamped and $y=0$ and $y=2\pi$
are periodic.
Figure is adopted from \cite{RHB25}.
}\label{fig:trommel}
\end{figure}

We adapt an example from~\cite{RHB25} involving a computational mechanics model of a trommel screen (whose finite element mesh is illustrated in the left plot of Figure~\ref{fig:trommel}) to investigate a practical scenario in which information is accumulated through a sequence of independent experiments.
Due to their operational use in separating materials of distinct sizes in mineral and solid-waste applications, the structural health of a population of screens (defined by the distribution of Young's modulus) should be periodically analyzed.
We simulate three separate experiments involving the separation of large, moderate, and small debris for which displacement data are recorded (with precision limited to two decimal places) at the 63 sensor locations shown in the right plot of Figure~\ref{fig:trommel}.
At a high-level, each experiment provides distinct observational data from which separate data-derived QoI maps are learned through the LUQ (Learning Uncertain Quantities) framework presented in \cite{RHB25}.
This results in push-forward constraints defined on observable spaces of different dimensions.

Since such experiments are performed asynchronously in practice, we sequentially incorporate their push-forward constraints, thereby progressively enriching the feasible set associated with the GSIP.
Theorem~\ref{thm:mtd2_converges} and Corollary~\ref{cor:feasible_sets} motivate this progressive refinement strategy by establishing convergence under their stated assumptions.

The computational model, finite-element discretization, displacement data, and learned QoI maps are adapted directly from~\cite{RHB25}, allowing the present example to isolate the impact of the inverse methodology.
The principal modifications are summarized in Table~\ref{tab:trommel_comparison}.
In particular, we introduce a correlated data-generating distribution, formulate the inverse problem as a progressively enriched GSIP, and recursively update the reference measure between refinement stages.
These changes allow the example to simultaneously illustrate the flexibility of Algorithm~\ref{alg:main} in incorporating heterogeneous QoI constraints and the progressive refinement strategy established by Corollary~\ref{cor:feasible_sets}.

\begin{table}[ht]
\centering
\caption{Comparison of the key component differences utilized in the present work from the inverse methodologies employed in the LUQ trommel-screen study of~\cite{RHB25} so that the influence of the modified inverse methodology can be isolated.}
\label{tab:trommel_comparison}

\begin{tabular}{p{4.0cm}p{5.0cm}p{5.6cm}}
\toprule
\textbf{Key Component} &
\textbf{LUQ Study~\cite{RHB25}} &
\textbf{Present Work} \\
\midrule

DG distribution &
Independent with
$\mathrm{Beta}(2,6)$,
$\mathrm{Beta}(3,4)$,
$\mathrm{Beta}(2,4)$
marginals &
Gaussian copula with $\mathrm{Beta}(4,6)$,
$\mathrm{Beta}(8,3)$,
$\mathrm{Beta}(3,9)$
marginals \\
\hline

Solution strategy &
Separate SIPs solved by three DCI updates &
Progressively enriched GSIP solved by iDCI \\
\hline

{Utilization of accumulated experimental information} &
{Each DCI update uses only the current experiment; previously enforced push-forward constraints are discarded} &
{Each refinement stage retains all previously imposed push-forward constraints to define a progressively enriched GSIP} \\
\bottomrule
\end{tabular}
\end{table}

In this work, the Young's modulus values in the three separately shaded sections of the trommel screen shown in the left plot of Figure~\ref{fig:trommel} are labeled, left-to-right, as $\lambda_1$, $\lambda_2$, and $\lambda_3$.
As in \cite{RHB25}, it is further assumed that the parameter space is defined as
\[
    \pspace := [0.9,1.1]\times[0.72,0.88]\times[0.576, 0.705]
\]
with a uniform initial distribution.
The marginal Beta distributions described in Table~\ref{tab:trommel_comparison} are appropriately shifted and scaled as necessary for each parameter dimension.
As in \cite{RHB25}, we utilize a total of 500 initial samples and 50 data-generating samples.

We assume the large, medium, and small debris experiments are conducted in that order and contribute learned QoI spaces of dimensions 2, 3, and 1, respectively.
Thus, $\p_1$ is defined, equivalently, either by the single estimated observed density on the two-dimensional QoI space learned from the large debris experiment or by the two observed marginal densities along with their copula.
We opt for the former so that the solution is given by a standard application of DCI.
Then, we estimate the associated joint observed density on the next three-dimensional space to define the second constraint for the enriched feasible set $\p_2$.
Since we do not assume that any cross-experiment copula information is known between the separately conducted debris experiments, the solution for $\p_2$ is constructed by applying the iDCI approach of Algorithm~\ref{alg:main} to the two- and three-dimensional QoI spaces learned from the separate experiments using the solution obtained for $\p_1$ as the initial reference measure.
The observed density for the final one-dimensional QoI is then estimated to construct the third constraint used in the final enriched feasible set $\p_3$.
The final solution is obtained by again applying the iDCI method to $\p_3$ starting from the solution to $\p_2$.
Results are summarized with 2D marginal plots for each feasible set in Figure~\ref{fig:trommel_results}, which illustrate that progressive enrichment of the feasible sets produces solutions that increasingly improve in their ability to recover the dependence structure within the data-generating distribution.

\begin{figure}
    \centering
    \includegraphics[width=0.9\linewidth]{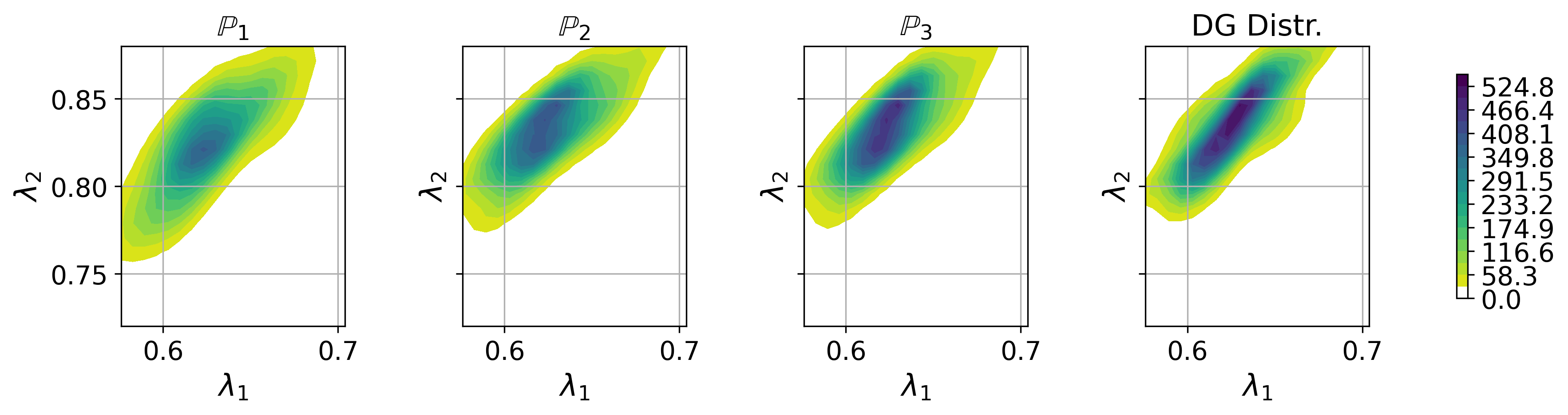}\\
    \includegraphics[width=0.9\linewidth]{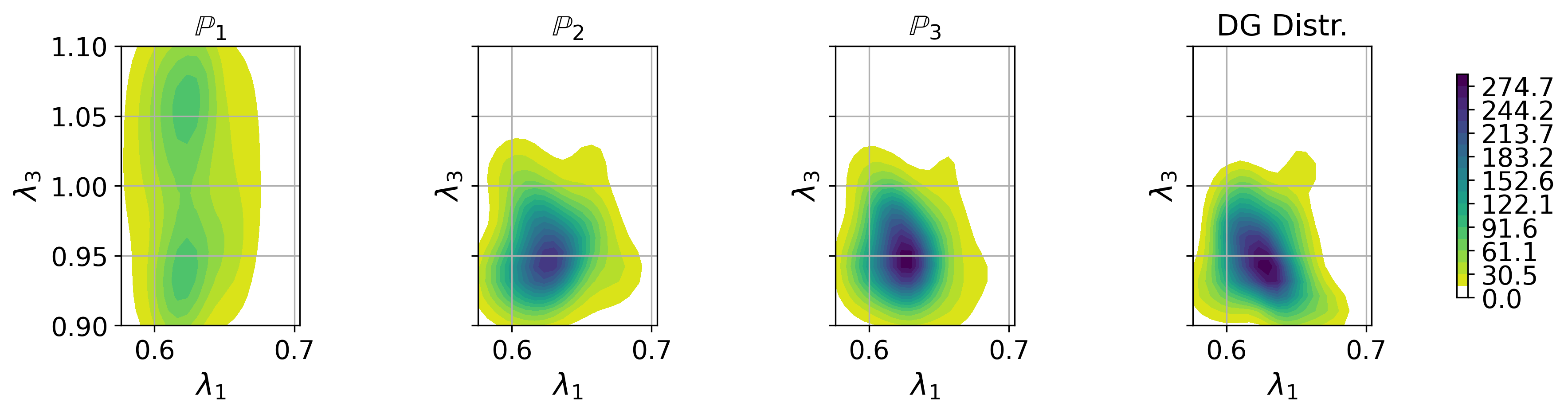}\\
    \includegraphics[width=0.9\linewidth]{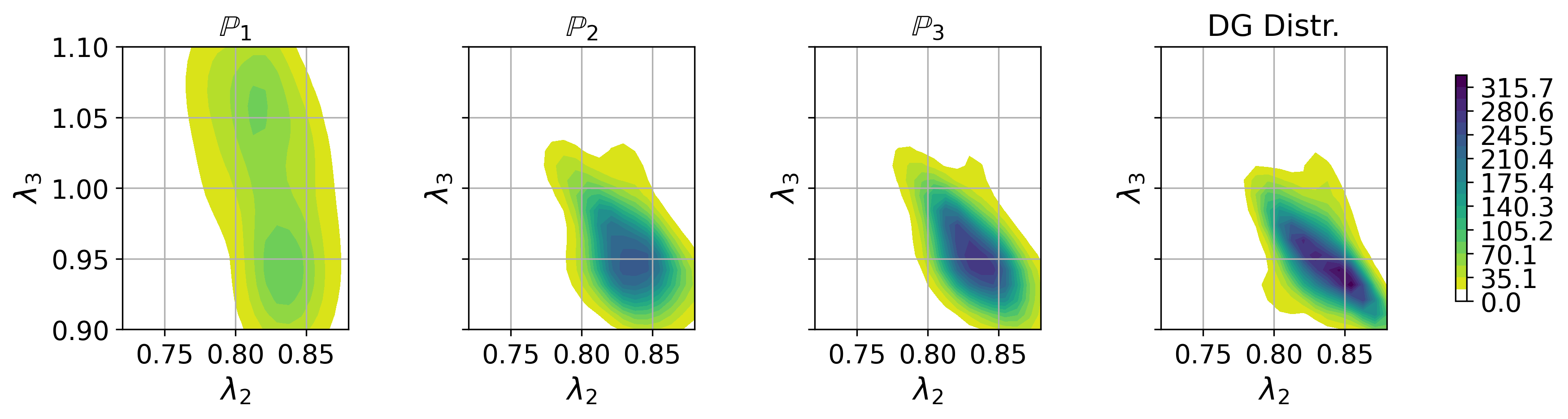}
    \caption{Top, middle, and bottom rows show, respectively, the $(\lambda_1,\lambda_2)$, $(\lambda_1, \lambda_3)$, and $(\lambda_2, \lambda_3)$ marginal plots of the progressively enriched GSIP solutions for the trommel screen example.
    The leftmost column shows the marginals associated with the $\p_1$ feasible set defined by the two-dimensional push-forward constraint associated with the large debris experiment.
    The second-to-left column is for $\p_2$ defined by enriching the $\p_1$ feasible set with the three-dimensional push-forward constraint associated with the medium debris experiment.
    The second-to-right column is for $\p_3$  defined by enriching the $\p_2$ feasible set with the one-dimensional push-forward constraint associated with the small debris experiment.
    The sample data-generating distribution marginals are shown in the rightmost column.}
    \label{fig:trommel_results}
\end{figure}

The quantitative results in Table~\ref{tab:trommel_kl} summarize the progressive refinement procedure.
The column, \emph{Accumulated QoI}, records the cumulative dimensions of the learned QoI spaces incorporated up to each stage that are used to enrich the associated feasible set.
The final row reports the direct iDCI solution on the fully enriched feasible set $\p_3$, computed from the original uniform initial distribution, for direct comparison with the progressively refined solution.

\begin{table}[ht]
\centering
\caption{KL divergences from the sample data-generating (DG) distribution for the trommel-screen example.
The column \emph{Accumulated QoI} records the cumulative dimensions of the learned QoI spaces incorporated up to each refinement stage (2, then 2+3, then 2+3+1).
The final row reports the direct iDCI solution on the fully enriched feasible set $\p_3$, computed from the original uniform initial distribution, for comparison with the progressively refined solution.}
\label{tab:trommel_kl}
\small
\begin{tabular}{p{4.4cm}p{3.5cm}p{3cm}p{2.6cm}}
\toprule
\textbf{Stage / Feasible Set} & \textbf{Reference Measure} & \textbf{Accumulated QoI} & \textbf{KL Divergence} \\
\midrule
Initial distribution (KDE) & Uniform initial & -- & $\expnumber{2.9185}{0}$ \\
Stage 1: $\p_1$ (DCI) & Uniform initial & 2 & $\expnumber{1.3355}{0}$ \\
Stage 2: $\p_2$ (iDCI) & Solution on $\p_1$ & 2+3 & $\expnumber{3.452}{-1}$ \\
Stage 3: $\p_3$ (iDCI) & Solution on $\p_2$ & 2+3+1 & $\expnumber{3.003}{-1}$ \\
Direct iDCI on $\p_3$ & Uniform initial & 2+3+1 & $\expnumber{3.0083}{-1}$ \\
\bottomrule
\end{tabular}
\end{table}

The table shows that progressively enriching the feasible set rapidly reduces the KL divergence from the sample DG distribution, decreasing from $\expnumber{2.9185}{0}$ for the initial estimate to $\expnumber{3.003}{-1}$ after the third refinement stage.
Moreover, the direct iDCI solution on the fully enriched feasible set $\p_3$ is practically indistinguishable from the progressively refined solution.
Computing the KL divergence of this direct iDCI solution from the progressively refined solution produces a value of approximately $\expnumber{3.891}{-4}$.
This is consistent with the behavior predicted by Corollary~\ref{cor:feasible_sets} as retaining the previously learned push-forward constraints while recursively updating the reference measure is expected to yield the same limiting solution as solving the fully enriched GSIP directly.

To contextualize these values, we also estimated the data-generating distribution directly from fifty independent sets of $50$ samples.
The KL divergences between the sample data-generating distribution used and these separate kernel density estimates had an average of approximately $\expnumber{5.395}{-1}$, with a minimum of $\expnumber{2.147}{-1}$ and a maximum of $\expnumber{9.665}{-1}$.
Thus, the progressively refined solution and the direct iDCI solution both lie well within the variability expected from direct estimates of the data-generating distribution with this sample size.

\section{Conclusions and Future Work}\label{sec:conclusions}

This work establishes a theoretical and computational bridge between the DCI and iDCI frameworks through copula theory.
By applying Sklar's theorem to the observed and predicted push-forward distributions, we derived a factorization of the DCI update into separate marginal and dependence transformations and showed that the discrepancy remaining after convergence of the iDCI algorithm is entirely characterized by the corresponding copulas.
This characterization motivated a copula-transformed iDCI solution, for which we proved that an exact copula transformation recovers the original DCI solution.
We further developed convergence theory for approximate copula transformations under converging reference measures and progressively enriched feasible sets, providing theoretical support for adaptive computational strategies.
The numerical examples demonstrated that the importance of the copula transformation is governed primarily by the geometry induced by the QoI map, that adaptive reference-measure refinement substantially improves approximation quality while maintaining a fixed per-stage computational budget, and that progressively enriched GSIPs naturally incorporate heterogeneous, asynchronously acquired experiments.

Future work will investigate both theoretical and computational extensions of the framework developed here.
A particularly promising direction is to better understand how the geometry induced by the QoI map governs the magnitude of the dependence transformation and how this relationship may be exploited for optimal experimental design (OED), sensor placement, and the construction of learned QoI maps that naturally minimize the need for copula transformations.
Such developments would complement prior DCI-based OED analyses (e.g.~\cite{WWJ17, BJP+26}) and would enable experimental campaigns in which heterogeneous or asynchronously acquired observations can be incorporated with minimal loss of information despite the absence of paired measurements.
On the computational side, a systematic investigation of copula estimation strategies, including nonparametric estimators, vine copulas, and normalizing-flow-based copulas, is needed to better understand the trade-offs between model complexity, computational cost, and approximation accuracy.
Finally, the progressive refinement framework developed here naturally suggests broader adaptive methodologies for GSIPs in which increasingly informative observable spaces and additional probabilistic constraints are incorporated sequentially as new experimental information becomes available.

\section{Data Availability and Supplementary Material}\label{sec:supplementary}

The public repo \verb|https://github.com/CU-Denver-UQ/Iterative-DCI| contains the code and datasets utilized to generate the results for the numerical examples as well as a summary of package dependencies required to execute the code provided.
Executing the Jupyter notebooks associated with the first two examples will both generate all the data as well as the figures shown in Section~\ref{subsec:copulas_linear_example} and \ref{subsec:adaptive_measures}.
The simulation data for the trommel screen example shown in Section~\ref{subsec:trommels} are also provided within the repo and are utilized when executing the associated Jupyter notebook to generate the results reported within this manuscript.

\section{Acknowledgments}\label{sec:acknowledgments}
T.~Butler and J.~Silva's work is supported by the National Science Foundation under Grant No.~DMS-2208460.
However, any opinion, finding, and conclusions or recommendations expressed in this material are those of the author and do not necessarily reflect the views of the National Science Foundation.

H.~Hakula is supported by the Research Council of Finland (Flagship of Advanced Mathematics for Sensing Imaging and Modelling grant 359181).

This paper describes objective technical results and analysis. Any subjective views or opinions that might be expressed in the paper do not necessarily represent the views of the U.S. Department of Energy or the United States Government.
Sandia National Laboratories is a multimission laboratory managed and operated by National Technology and
Engineering Solutions of Sandia, LLC., a wholly owned subsidiary of Honeywell International, Inc.,
for the U.S. Department of Energy’s National Nuclear Security Administration under contract DE-NA-0003525.  This material is based upon work supported by the U.S. Department of Energy, Office of Science, Office of Advanced Scientific Computing Research, under contract 25-025287.

The authors thank Professor Haonan Wang for the valuable discussions.

\section{Generative AI Disclosure}

During the preparation of this work, the authors used a University of Colorado Denver licensed version of ChatGPT to assist with editing, proofreading, and improving the clarity and presentation of the manuscript. ChatGPT was also used to generate initial drafts of some helper functions involving KL divergence calculations and Gaussian copula estimation in the code used to produce the numerical results; these functions were subsequently reviewed and modified by the authors to meet the specific requirements of the surrounding code. In all cases, the authors reviewed and edited the AI-assisted content as needed and take full responsibility for the content of the publication and accompanying code.

\bibliographystyle{IJ4UQ_Bibliography_Style}
\bibliography{references}
\end{document}